\documentclass{article} % For LaTeX2e
\usepackage{iclr2026_conference,times}

\usepackage{amsmath,amsfonts,bm}

\def\eqref#1{equation~\ref{#1}}
\def\1{\bm{1}}

\DeclareMathAlphabet{\mathsfit}{\encodingdefault}{\sfdefault}{m}{sl}
\SetMathAlphabet{\mathsfit}{bold}{\encodingdefault}{\sfdefault}{bx}{n}

\usepackage{hyperref}
\usepackage{url}
\usepackage{xspace}
\usepackage{graphicx}

\usepackage{siunitx}
\usepackage{wrapfig}
\usepackage{lipsum} % for dummy text
\usepackage{multirow}
\usepackage{booktabs}

\usepackage{subcaption}
\usepackage{caption}

\usepackage{makecell}
\usepackage{pifont} % for checkmarks / xmarks
\newcommand{\cmark}{\ding{51}}
\newcommand{\xmark}{\ding{55}}

\usepackage{enumitem}
\usepackage{longtable}

\newcommand{\sysname}[0]{{\sc ProRe}\xspace}

\newcommand{\ie}{{\it i.e.,}\xspace}

\usepackage{listings}
\usepackage[utf8]{inputenc}
\usepackage{amsthm}

\newtheorem{lemma}{Lemma}               
\theoremstyle{remark}

\title{\sysname: A Proactive Reward System for GUI Agents via Reasoner–Actor Collaboration}

\author{Gaole Dai\thanks{Work done during internship at Microsoft Research.} \\
NTU \\
\And
Shiqi Jiang\thanks{Corresponding authors}\\
Microsoft Research \\
\And
Ting Cao \\
Tsinghua University \\
\And
Yuqing Yang \\
Microsoft Research \\
\And
Yuanchun Li \\
Tsinghua University \\
\And
Rui Tan\\
NTU \\
\And
Mo Li\footnotemark[2]\\
HKUST \\
\And
Lili Qiu \\
Microsoft Research \\
}

\iclrfinalcopy % Uncomment for camera-ready version, but NOT for submission.
\begin{document}

\maketitle

\begin{abstract}
Reward is critical to the evaluation and training of large language models (LLMs). However, existing rule-based or model-based reward methods struggle to generalize to GUI agents, where access to ground-truth trajectories or application databases is often unavailable, and static trajectory-based LLM-as-a-Judge approaches suffer from limited accuracy. To address these challenges, we propose \sysname, a proactive reward system that leverages a general-purpose reasoner and domain-specific evaluator agents (actors). The reasoner schedules targeted state probing tasks, which the evaluator agents then execute by actively interacting with the environment to collect additional observations. This enables the reasoner to assign more accurate and verifiable rewards to GUI agents. Empirical results on over 3K trajectories demonstrate that \sysname improves reward accuracy and F1 score by up to 5.3\% and 19.4\%, respectively. Furthermore, integrating \sysname with state-of-the-art policy agents yields a success rate improvement of up to 22.4\%. The source code is available at \url{https://github.com/V-Droid-Agent/ProRe}.

\end{abstract}
\section{Introduction}
\label{sec:intro}

% \jsq{Reward system is essential to enable self-involving GUI agents with in the era of experience}

% Recently, verifiable rewards have attracted increasing attention in the research community, as they have proven to be a crucial component in the large-scale training of large language models (LLMs) and vision-language models (VLMs). In this paradigm, LLMs/VLMs serve as the policy models, attempting to complete user queries by generating analyses, calling functions, or interacting with graphical user interfaces (GUI). Rewards are then assigned by evaluating whether the tasks have been successfully accomplished. These successful answers or trajectories can be leveraged to curate high-quality training datasets, supporting a variety of training objectives. Moreover, verifiable rewards also provide a standardized mechanism for benchmarking the performance of policy models.

Verifiable rewards are pivotal for enabling the continual evolution of large language model (LLM)-based agents~\cite{wang2024benchmark, guo2025deepseek, silver2025welcome}. Within this paradigm, LLMs operate as policy networks, undertaking user requests to generate reasoning, invoke tools and functions, and manipulate graphical user interfaces (GUIs)~\cite{qi2024webrl}. Rewards function as quantitative feedback signals that steer the agent's learning process~\cite{gao2024designing}, promoting optimal behaviors while discouraging suboptimal actions.

Reinforcement learning with verifiable rewards (RLVR) has the potential to significantly advance GUI agents~\cite{wang2024distrl, xu2025mobilerlonlineagenticreinforcement, wang2025ui}. A simple yet effective binary reward for GUI automation is to assess whether the specified task has been successfully completed. To obtain such a reward signal, existing methodologies could be generally categorized into rule-based and LLM-based, as illustrated in Figure~\ref{fig:teaser}. In the rule-based paradigm, human experts manually construct verification code snippets to ascertain the realization of the intended state for each task. For instance, AndroidWorld~\cite{rawles2024androidworld} and WindowsAgentArena~\cite{bonatti2024windows} datasets contain more than 116 and 150 manually engineered unit testing code, respectively, to provide grounded signals of task accomplishment for individual GUI automation tasks. While this approach offers high accuracy, it is inherently limited in scalability, as the manual creation of unit testing scripts demands substantial human effort and resources, thereby preventing its use as a reward mechanism for large-scale GUI agent training.

LLM-as-a-judge is thus proposed to enable scalable agentic rewards \cite{gu2024survey, bai2024digirl}. Leveraging the capabilities of advanced LLMs such as GPT-4o, this approach evaluates GUI task trajectories, often represented as screenshots, by prompting the model with queries such as, \textit{“Based on the task trajectory, please determine if the task is completed”}. LLM-as-a-judge offers an autonomous and scalable framework for allocating reward signals \cite{wang2024distrl}. However, we observe that this approach is considerably less effective for rewarding GUI agents.

The rationale underlying the failures of LLM-as-a-judge for GUI agents is twofold: \emph{incomplete state observability} of GUI tasks and \emph{limited domain-specific capabilities} of LLMs.

First, GUI task states are typically monitored \emph{passively} through specific modalities, such as screenshots \cite{gou2024navigating}. However, owing to the inherent complexity and dynamic nature of GUI interactions, these states frequently remain incompletely observable. For instance, as depicted in Figure~\ref{fig:teaser}, during the monitoring of the \textit{“taking two photos”} task exclusively through screenshots, the captured lacks critical success indicators, thereby precluding even human evaluators from reliably ascertaining task completion. Moreover, observations are typically conducted at fixed intervals, potentially omitting critical state transition details. Consequently, GUI state observability remains inherently incomplete, thereby compromising the efficacy of the reward system.

% Trajectory-based LLM-as-a-Judge methods rely heavily on generalist LLMs/VLMs, which often lack domain-specific knowledge of user interfaces (UIs) and applications. The complexity and richness of trajectory information, combined with potentially misleading actions from GUI agents, can easily confuse the judge. Furthermore, the judge’s performance is constrained by partial observations of the environment. For simplicity and clarity, UI designs typically hide certain key information to fit within limited screen space and enhance user experience. Meanwhile, screenshots are collected either at fixed sampling rates or passively, which often results in missing critical dynamic information in the environment. To address these limitations, our key idea is to enable the agent itself to proactively verify the environment state, rather than relying on passive judgments of static trajectories.

\begin{figure}[!t]
    \centering
    \includegraphics[width=.95\linewidth]{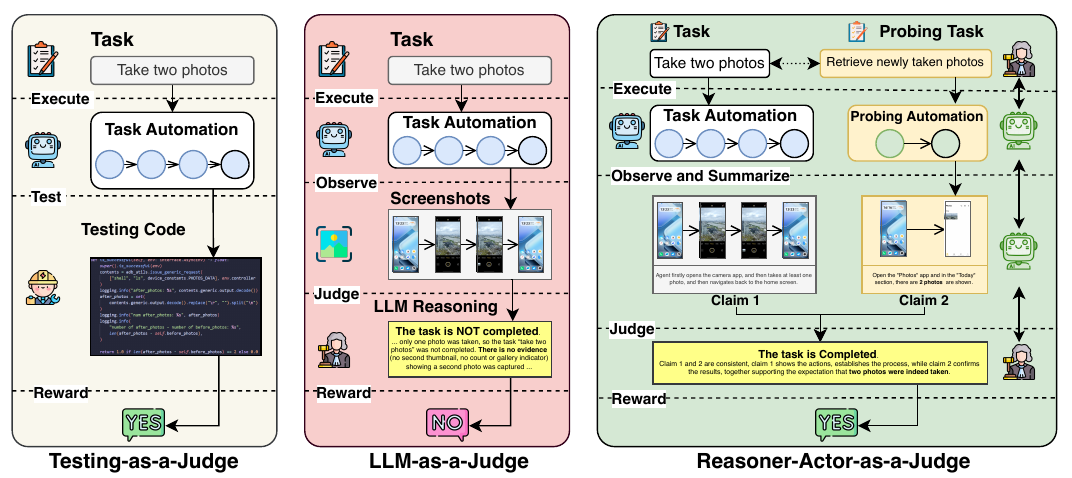}
    \caption{\sysname proposes to reward GUI agents using reasoner-actor-as-a-judge, rather than relying on expert to hand craft testing code or LLM to judge static trajectories.}
    \label{fig:teaser}
    \vspace{-1.5em}
\end{figure}

Second, evaluating GUI task states requires domain-specific GUI knowledge and expertise, which general-purpose LLMs utilized in reward systems, such as GPT-4o and Gemini, fundamentally lack~\cite{dai2025advancing}. Most general-purpose LLMs demonstrate suboptimal performance on GUI-related tasks~\cite{qin2025ui}. Although post-training may enhance their domain-specific proficiency, training of a domain-specific reasoner as the reward model still necessitates annotated datasets, thereby constraining its scalability. Consequently, deploying a general-purpose LLM to assess intricate domain-specific details intrinsically undermines the efficacy of the reward system.

To develop a scalable and accurate reward for GUI agents, this paper introduces \sysname, a proactive reward system based on reasoner–actor collaboration. The key idea of \sysname is to introduce the additional \emph{state probing} tasks planned by the reasoner. These tasks are executed by domain-specific evaluator agents (actors) that interact with the environment to retrieve key states relevant for task verification. Instead of relying solely on the policy agent’s execution trajectory, \sysname assigns rewards through high-level reasoning over the outcomes of these probing tasks.

% To develop a scalable and accurate reward for GUI agents, this paper introduces \sysname. The key idea of \sysname is to introduce the additional \emph{state probing} tasks, executed by evaluator agents with GUI expertise to look for key states. Rather than relying solely on the policy agent’s execution trajectory, \sysname allocates rewards based on high-level reasoning over the outcomes of these probing tasks.

Specifically, the reasoner, \ie GPT-4o, schedules the state probing tasks, conditioned on the original task objective and its expected outcome. After the policy agent finishes execution, evaluator agents are invoked to automate these probing tasks. They then summarize both the original task trajectory and the probed UI states into high-level, verifiable claims. The reasoner performs final judgment through chain-of-claims reasoning, which analyzes the \emph{consistency} between the policy agent’s claims and those generated from the evaluators' probing. An intuitive example is illustrated in Fig.~\ref{fig:teaser}: given the original task \textit{"taking two photos"}, a probing task \textit{"retrieving newly taken photos"} is formulated, with the expected outcome that \textit{"two photos should have appeared in the gallery"}. The evaluator agent executes this probing task and observes that \textit{"there are two newly captured photos from 11:00 AM to the present"}. The reasoner then assesses the consistency between the claims, thereby probably concluding that the original task has been successfully accomplished.

% Specifically, the state probing task is scheduled by a reasoner, \ie GPT-4o, conditioned on the original task objective and its expected outcome. The evaluator agents then automates those probing tasks after the execution of policy agents and summarize the policy agents' trajectories and additional probed states into high-level claims. The reasoner renders a final judgment with chain-of-chaims reasoning, which analyze the \emph{consistency} between the policy and evaluator claims. An intuitive example is depicted in Fig.~\ref{fig:teaser}. Given the original task \textit{"taking two photos"}, a state probing task, \textit{"retrieving newly taken photos"}, is formulated, along with the expectation that \textit{"two photos should have appeared in the gallery"}. The evaluator agent subsequently executes this probing task, resulting in the outcome \textit{"there are two newly captured photos from 11:00 AM to the present"}. The reasoner then assesses the consistency between the claims, thereby probably concluding that the original task has been successfully accomplished.

These designs address the fundamental challenge of rewarding for GUI agents in the following ways: 1) \sysname transforms the reward system from passive monitoring to proactive probing. The introduction of state probing tasks provides a complementary perspective to ascertain whether the original task has been accomplished; 2) \sysname decouples the general-purpose reasoner from domain-specific GUI judgments. Domain-specific actions are executed by domain-specific actors (evaluator agents), while the general-purpose reasoner concentrates solely on high-level logical consistency verification, which falls within the core competencies of general-purpose LLMs; 3) \sysname introduces a unique opportunity for co-evolution between the policy agent and the reward system. The execution of state probing tasks can be further optimized in tandem with the evaluator (policy) agent’s improvement, enabling a more sophisticated reward system that, in turn, facilitates accelerated progress for the policy agent.

We evaluate the performance of \sysname on typical GUI tasks. Specifically, \sysname is evaluated on over 3K distinct task traces collected from three benchmarks: AndroidWorld \cite{rawles2024androidworld}, AndroidLab \cite{xu2024androidlab}, and MobileAgentBench \cite{wang2024mobileagentbench}. The results demonstrate that, compared to existing state-of-the-art LLM-as-a-Judge approaches, \sysname enhances reward accuracy and F1 score by up to 5.3\% and 19.4\%, achieving an average accuracy of 93.7\%, thereby becoming the first reward system to surpass 90\% reward accuracy. In addition, pilot experiments on OSWorld and OSWorld-Chrome \cite{xie2024osworld} show that \sysname improves reward accuracy by 4.0\% on PC tasks and 6.5\% on web tasks. Moreover, when incorporated into policy agents to guide their test-time scaling strategy, \sysname elevates the success rate by at most 22.4\%. 

In summary, the key contributions of this works are as follows:

% With the additional evidence collected, the proposed reward system achieves significantly higher reward accuracy across diverse dynamic benchmarks and policy agents. In particular, its evaluation accuracy improves by 10.3\%, 8.4\%, and 9.5\% on the trajectories of state-of-the-art (SOTA) agents on the AndroidWorld, AndroidLab, and MobileAgentBench datasets, respectively. Furthermore, when integrated into V-Droid to guide its test-time scaling process, \sysname boosts the success rate of V-Droid to 68.1\%, establishing a new state-of-the-art performance on the AndroidWorld benchmark. The key contributions of this work are as follows:

\begin{itemize}[left=0pt]
    \item  We systematically study and empirically demonstrate the limitations of existing trajectories-based LLM-as-a-judge for GUI agents.
    \item We propose \sysname, a proactive reward system with a general reasoner that performs high-level scheduling and reasoning and domain-specific evaluator agents that actively probe states.
    \item \sysname achieves consistently higher reward accuracy and F1 score on different agents and benchmarks, and significantly improves the success rate of policy agents through test-time scaling.
\end{itemize}
% 1) We systematically study and empirically demonstrate the limitations of existing reward models for GUI agents that rely on static agent trajectories.

% 2) We propose \sysname, a proactive reward system with a generalist reasoner that performs high-level reasoning and evaluator agents that actively interact with the environment to collect verifiable evidence.

% 3) \sysname achieves consistently higher reward accuracy across multiple agents and benchmarks, and significantly improves the success rate of SOTA agents through test-time scaling.
\section{Related Works}

% \subsection{From Reward Models to Reward Systems}
% General reward models are widely used to support the experience-based training of LLMs, without the need of pre-collected ground truth or constructing rules with domain expertise \cite{gu2024survey, son2024llm}. Existing general reward models assign absolute scores to each answer or relative scores comparing answer pairs \cite{lin2025learning, xiong2025stepwiser, liu2025inference}. Some works build reward systems with agent-as-a-judge, which leverages tools, such search, executing codes or reading materials, to assist the reward generation \cite{zhuge2024agent, yu2025ais}. However, those reward systems cannot be applied to reward GUI agents in the wild, which executes different types of tasks that cannot be handled with predefined tool boxes.

\subsection{General Reward Models in Broader Topics}
General reward models are widely used to support experience-based training of LLMs, without requiring pre-collected ground truth or handcrafted rules from domain experts \cite{gu2024survey, son2024llm}. Such models typically assign either absolute scores to individual answers or relative scores by comparing answer pairs \cite{lin2025learning, xiong2025stepwiser, liu2025inference}. Beyond these, some works have proposed building \emph{reward systems} through the agent-as-a-judge paradigm, where agents are equipped with tools such as web search, code execution, or document reading to assist reward generation \cite{zhuge2024agent, yu2025ais}. However, these reward systems remain limited in scope and cannot be directly applied to GUI agents in the wild, which execute diverse task types that cannot be verified by a predefined toolbox.

\subsection{GUI Agents for Tasks Automation}
LLM-based GUI agents, which operate across websites, desktops, and smartphones to handle a wide spectrum of tasks ranging from professional work to everyday activities, have recently attracted significant attention \cite{lai2025computerrl, dai2025advancing, qin2025ui, gu2025ui, ye2025mobile}. LLMs are primarily employed either as generators to propose actions and decisions or as verifiers, to evaluate actions \cite{gou2024navigating, qin2025ui, liu2025infigui, dai2025advancing}. To improve the decision-making ability of GUI agents, various training paradigms—including supervised fine-tuning, direct preference optimization (DPO), and reinforcement learning—have been applied on large-scale datasets \cite{luo2025gui, tang2025gui, dai2025advancing, wang2024distrl}. Within this pipeline, accurate reward signals are crucial, as they enable automatic data collection at scale, which in turn underpins both dataset curation and model training \cite{tang2025gui, li2025mobileuse, qi2024webrl}.

% LLM-based GUI agents, which operates websites, desktops, and smartphones to handle a wide range of tasks from working to daily life activites, have attracted a lot of attention \cite{lai2025computerrl, dai2025advancing, qin2025ui, gu2025ui, ye2025mobile}. LLMs are used mainly used as generators to generate decision or verifiers to score actions. Besides, some works leverage grounding models to output the coordinates based on the analysis from the LLMs \cite{gou2024navigating, qin2025ui,liu2025infigui}. Different training methods, such as supervised fine-tuning, DPO and reinforcement learning, has been conducted to enhance the decision-making capabilities of GUI agents with large-scale dataset \cite{luo2025gui, tang2025gui, dai2025advancing, wang2024distrl}. Accurate reward empowers the automatic data collection pipeline, which is of vital importance for the large-scale dataset curation as well as the model training \cite{tang2025gui, li2025mobileuse, qi2024webrl}.

\subsection{Gaps between General Rewards and Rewards for GUI Agents}
There are some pioneering works on designing reward methods for GUI agents \cite{bai2024digirl, wang2024distrl, luo2025gui, lai2025androidgen}. They develop outcome or step-wise reward models to judge the success of GUI agents passively using the trajectories of GUI agents \cite{tang2025sea, hu2025guiding}. However, their performances are far from satisfying due to the partial observations of GUI agents to the states and the lack of domain knowledge of general-purpose LLM. One concurrent work, \cite{gou2025mind2web} constructs rubic trees for predefined web search tasks and checks key points with url, which lacks generalizability to in-the-wild tasks without such url. Instead, \sysname is the first reward system for GUI agent with a generalist reasoner to schedule state probing and evaluator agents to proactively probe states. 

% DistRL provides examples and screenshots to VLMs and prompt the reward models to compare expections with those observations from the trajectories. WebRL \cite{qi2024webrl} optimize the outcome reward model with supervised fine-tuning on self-curated website datasets. Step-Critic \cite{lai2025androidgen} propose to let LLM decompose users' goals into subtasks, and the reward model is instructed to check the finish status of each subtask.

% \dgl{introduce existing works in different types of rewards. 1) hard to scale for rule-based. 2) outcome reward follows llm-as-a-judge and rely on the traj itself; 3) progress reward is less practical and they rely on the outcome reward}
% \dgl{we focus on outcome reward model in this work}

\section{Proactive Reward System with Agent-in-the-loop} \label{sec:method}
% \label{headings}
\subsection{Problem Formulation.} Given the users instruction $\mathcal{G}$, a policy agent $\pi$ interacts with the environment consecutively, which forms a $N$ steps trajectory $\tau = (s_0, a_0, s_1, a_1, \dots, s_T)$. $s$ is the observation of $\pi$ on step $t$ and $a$ is the $t$-th actions. The goal is to generate an accurate binary outcome reward $r$ on $\tau$. \label{sec:problem_formulation}

\begin{lemma}
\label{lem:reward_success}
Let the success rate of the policy agent be $p_a$ and the reward accuracy be $p_c$. 
Then, under test-time scaling with trial budget $N$, the final success rate $P_{final}$ satisfies
\[
P_{final} = \frac{p_ap_c}{q}\big[1-(1-q)^N\big] + p_a(1-q)^N, 
\quad \text{where } q = p_ap_c+(1-p_a)(1-p_c).
\]
In particular, given $p_a > 0$, $P_{final}$ monotonically increases with respect to $p_c$ whenever $p_c > 0.5$.
\end{lemma}

A full proof is deferred to Appendix \ref{appendix:tts}. 
Our work focuses on improving $p_c$ and $P_{final}$. 

\begin{figure}[!t]
    \centering
    \begin{minipage}{0.50\linewidth}
        \centering
        \includegraphics[width=\linewidth]{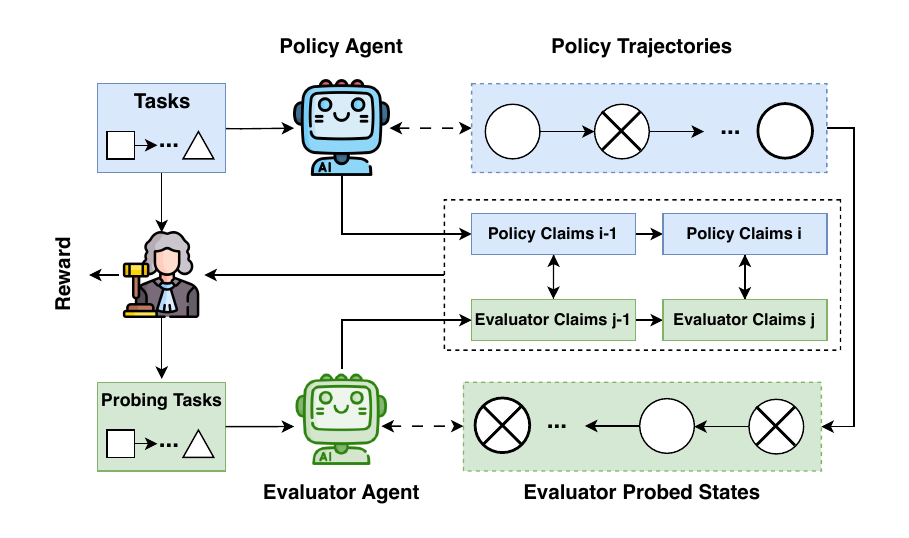}
        \caption{\sysname overview.}
        \label{fig:framework}
    \end{minipage}\hfill
    \begin{minipage}{0.45\linewidth}
        \centering
        \includegraphics[width=\linewidth]{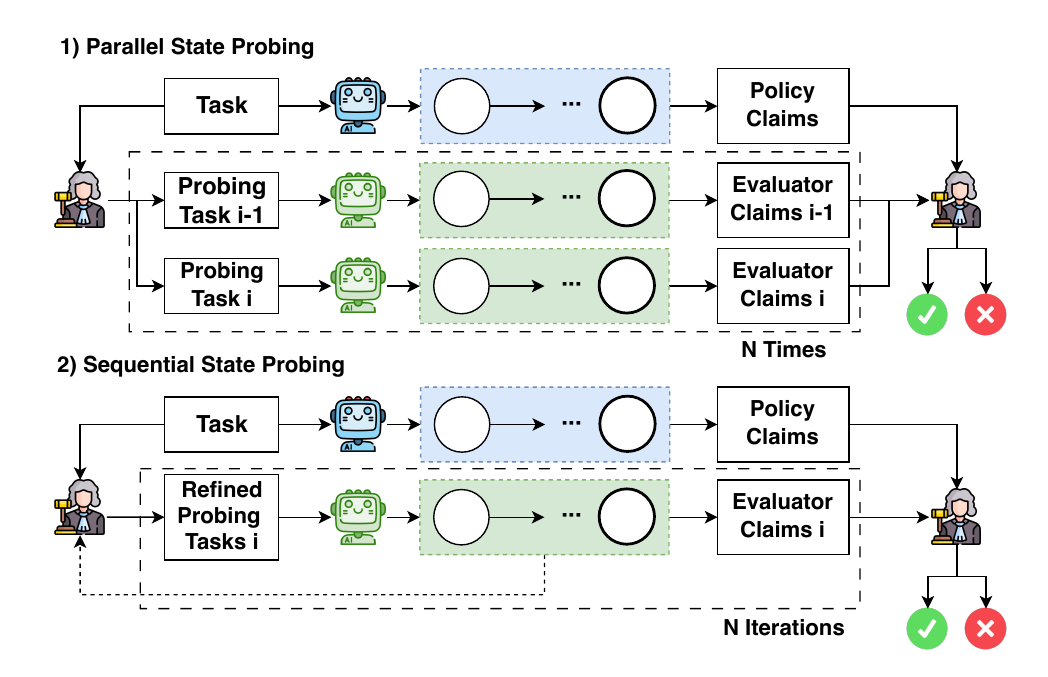}
        \caption{Test-time Scaling of \sysname.}
        \label{fig:tts}
    \end{minipage}
    \vspace{-1.5em}
\end{figure}

\subsection{Framework Overview}
 
% \jsq{Try to highlight 1) decouple common reasoner and specific actor, 2) turn domain-specific reasoning to generic casual verification}

Instead of applying LLM-as-a-Judge to generate a reward $r$ from trajectories, 
\sysname introduces a general LLM reasoner $\mathcal{J}$ working in collaboration with domain-specific evaluator agents $\pi_e$ for state probing, as shown in Figure~\ref{fig:framework}. Given the original tasks, the reasoner $\mathcal{J}$ first schedule probing tasks for the evaluator agents. Then the evaluators $\pi_e$ further explore the environment to collect key state information. The policy agent’s trajectories and the probed states are then summarized by $\pi_e$ into claims about task progress. 
Finally, the reasoner $\mathcal{J}$ analyzes the relationships and consistency among these claims, performing chain-of-claims reasoning to generate the outcome reward.

% Instead of directly leveraging LLM-as-a-judge to generate reward $r$ on the trajectories, \sysname is formulated with a general-purpose LLM reasoner $\mathcal{J}$ with domain-specific actors, or say, evaluator agents $\pi_e$, for state probing. The reasoner $\mathcal{J}$ first scatches state probing tasks for the evaluator agents given the original tasks. Based on those state probing tasks, $\pi_e$ further explore the environment to probe key information on the states. The policy agents trajectories and the probing states are then summarized by $\pi_e$ into chaims about tasks progress. Finally, the reasoner $\mathcal{J}$ generates the outcome reward with chain-of-claims by analyzing the inter-relationships and the consistency between claims.

% 
% \dgl{Introduce the details of the framework, including the status monitor to filter task-related information, the evaluation tasks generation, the proactive evidence search, and the final reward generation in two ways.}

% \dgl{Discribe how the prompt is formulated. Input-output. give some exmaples to help understand. Use figure to illustarte the process?}

% Besides, the evaluation task can be executed right after the policy agent finishes, without re-navigating from scratch. 
\subsection{Proactive Agent-in-the-loop Probing} \label{sec:probing}
% The partial observation prevents the LLM-as-a-judge to make accurate decisions. Specifically, existing methods uses the action histories and working memory from the policy agents traces as input, which are originally designed for continuous interaction and progress tracking rather than outcome judgement.

The partial observation to the GUI task states by the policy agents prevents LLM-as-a-judge to make accurate decisions. To handle this problem, \sysname introduces a set of evaluator agents to proactively probe states and collect additional information. The general-purpose LLM first schedule the state probing tasks for the evaluator agents based on the tasks inputs.

\textbf{State Probing Tasks Scheduling.} The general-purpose LLM reasoner is instructed to analyze the expectations and requirements specified in the original user instructions $\mathcal{G}$, and to identify the key states necessary for judging task success. 
Based on this analysis, the reasoner formulates state probing tasks, which are then issued as instructions for the evaluator agents to retrieve the corresponding key states from the environment.
\begin{equation}
   G_e \sim \mathcal{J}(G \mid \mathrm{Exp}, \mathcal{E}, L), 
   \quad \mathrm{Exp} = \mathcal{J}(G), 
   \quad G \in \mathcal{G}.
\end{equation}
where $\mathrm{Exp}$ is the analyzed expectations for the task $G$, $\mathcal{E}$ refers to the few shot examples provided and $L$ is the summarized guidelines for mapping the tasks to the probing tasks. To illustrate, when instructing the policy agent to delete a file \textit{A}, the corresponding state probing task $G^e$ is to search whether \textit{A} still exists in the target applications. The generation of expectations and state-probing tasks primarily rely on the reasoning capability of general-purpose LLM on analyzing the users expectations without the need of much domain-specific knowledge of APP and UI interactions. More examples on the state probing tasks are provided in Appendix~\ref{appendix:probing_tasks_example}.

% which are offloaded to evaluator agents.

% \begin{table}[!t]
% \centering
% \renewcommand{\arraystretch}{1.2}
% \begin{tabular}{l l c c}
% \toprule
% \textbf{Benchmark} & \textbf{Task Type} & \textbf{SR} & \textbf{Steps} \\
% \midrule
% \multirow{2}{*}{AndroidWorld}
%     & Dual-OB & \textbf{66.7\%} & \textbf{6.2} \\
%     & Others                & 53.6\%          & 14.7 \\
% \midrule
% \multirow{2}{*}{MobileAgentBench}
%     & Dual-OB & \textbf{64.0\%} & \textbf{6.8} \\
%     & Others                & 44.0\%          & 11.9 \\
% \midrule
% \multirow{2}{*}{AndroidLab}
%     & Dual-OB & \textbf{} & \textbf{} \\
%     & Others                &        &  \\
% \bottomrule
% \end{tabular}

\begin{table}[!t]
\centering
\caption{The probing tasks are generally easier than the execution tasks.}
\begin{tabular}{l c c c c c c}
\toprule
\multirow{2}{*}{\textbf{Task Type}} & \multicolumn{2}{c}{AndroidWorld} & \multicolumn{2}{c}{MobileAgentBench} & \multicolumn{2}{c}{AndroidLab} \\
\cmidrule(lr){2-3} \cmidrule(lr){4-5} \cmidrule(lr){6-7}
 & SR & Steps & SR & Steps & SR & Steps \\
\midrule
State Probing & 66.7\% & 6.2 & 64.0\% & 6.8 & 65.9\% & 6.1 \\
Execution           & 53.6\%         & 14.7         & 44.0\%         & 11.9         & 27.5\%          & 11.8 \\
\bottomrule
\end{tabular}
\label{tab:dualob_subcols}
\vspace{-1em}
\end{table}

% \begin{table}[!t]
% \centering
% \renewcommand{\arraystretch}{1.2}
% \begin{tabular}{lccc}
% \toprule
% \textbf{Task Type} & \textbf{AndroidWorld} & \textbf{MobileAgentBench} & \textbf{AndroidLab} \\
% \midrule
% \textbf{Dual-OB (SR / Steps)} & \textbf{66.7\% / 6.2} & \textbf{64.0\% / 6.8} & \textbf{72.0\% / 8.4} \\
% Others (SR / Steps) & 53.6\% / 14.7 & 44.0\% / 11.9 & 33.2\%/ 16.2 \\
% \bottomrule
% \end{tabular}
% \caption{Success Rate (SR) and average number of Steps for Dual-OB vs Others across datasets}
% \label{tab:sr_steps_by_task}
% \end{table}

\textbf{State Probing with Evaluator Agents.} Given the state probing tasks $G_e$, evaluator agents are provoked to interact with the environment in a step-wise manner to collect additional observations on key states right after the execution of the policy agent. 
\begin{equation}
    s^e_{t+1} = \mathcal{F}(s^e_t, a^e_{t}),\quad a^e_{t} = \pi_e(s^{\pi_e}_t, G_e)
\end{equation}
where $\mathcal{F}$ is the status transition of the environment; $s^e_{t}$ is the state of captured by the evaluator agents. The probing process mainly leverages the UI-related knowledge in the GUI agent while minimizing the requirements on its reasoning capability on understanding users expectations. 

% The state probing task $\mathcal{T}_e$ is observed to be easier compared with other kinds of tasks such as creation, status modification, or deletion. As shown in Table \ref{tab:dualob_subcols}, the success rate of the state probing tasks is 15\% higher than on other kinds of tasks and the average trajectory length is 40\% shorter. While all those tasks require understanding of UI layout and app design logics, state probing tasks only need the agent to navigate to the correct page and do not need consecutive error-free execution. Because state probing is generally easier than other types of tasks, the evaluator agent (and consequently the reward system) tends to be more generalizable than the policy agent. As a reason, the reward system can be used to guide the test-time scaling or training of the policy agent.

\textbf{The Execution-Probing Gap.} The state probing task $\mathcal{T}_e$ is generally easier than other types of execution tasks such as creation, status modification, or deletion. 
As shown in Table~\ref{tab:dualob_subcols}, V-Droid \cite{dai2025advancing} achieves a 23.8\% higher success rate on state probing tasks and the trajectories are 50.3\% shorter on average. While both probing and execution tasks involve knowledge on UI and applications, probing only requires navigating to the correct page and does not demand consecutive error-free execution. Because of this relative simplicity, the evaluator agent, and by extension the reward system, is more generalizable than the policy agent. This generalizability allows the reward system to effectively guide both the test-time scaling and the training of policy agents.

We also notice that there are some long-horizon tasks that demand checking multiple states across different pages. In those cases, the probing tasks could be formulated into multiple subtasks, based on which the evaluator agent execute sequentially to obtain complete probed states.

% \dgl{Show the observation that it is easier to search for the information compared with executing tasks. So we believe the reward system is more generalizable compared with the policy agent. We can first scale the reward system and then train the policy.}
\subsection{Outcome Reward with Chain-of-Claims}
\textbf{Chain-of-Claims.} To avoid overwhelming the general-purpose LLM with too much low-level GUI details in the probed states, the evaluator agents summarize the trajectories of the policy agent and the probed states into chain-of-claims. Specifically, given a trajectory $\tau$ generated by the policy agent $\pi$, the evaluator agent observes this sequence and the additional probed UI states to form claims about task progress. We define two sets of claims:

1) $\mathcal{C}^\pi = \{ c^\pi_1, c^\pi_2, \dots, c^\pi_{N_\pi} \}$: $N_\pi$ claims generated from the policy agent’s trajectory $\tau$.

2) $\mathcal{C}^{\pi_e} = \{ c^{\pi_e}_1, c^{\pi_e}_2, \dots, c^{\pi_e}_{N_{\pi_e}}\}$: $N_{\pi_e}$ claims made by the evaluator agents $\pi_e$. 

Each claim $c$ is structured as: 
\begin{equation}
    c=Claim(\tau_{g_j}= \{s_t, s_{t+1}, a_t\}), \text{where } g_j \in G
\end{equation}
where $\tau_{g_j}$ is a subtrajectory of the policy agent’s trajectory $\tau$ or a sequence of probed states produced by the evaluator agents. We instruct the evaluator agents to generate multiple claims covering different parts of the trajectory, which is observed to be more effective than segmenting trajectories via state clustering on learned embeddings.

Given these claims, the general-purpose LLM reasoner $\mathcal{J}$ performs chain-of-claims reasoning to produce the final reward by linking and comparing policy and evaluator claims:
\begin{equation}
    r = \mathcal{J}(G,\mathrm{Exp},\mathcal{C}),\quad
    \mathcal{C} = \left\{ c_i^\pi, c_j^{\pi_e}, r_{ij} \right\}
\end{equation}
where $r_{ij}$ denotes the relationship between a policy claim $c_i^{\pi}$ and an evaluator claim $c_j^{\pi_e}$. The relation can be confirming, contradicting, complementary, or unrelated.

\textbf{Claim Filters.} Irrelevant evidence and claims has the potential to compromise the accuracy of final judgments. Therefore, within the reasoner, we integrate a claim filter that explicitly identifies and eliminates irrelevant or misleading claims prior to the chain-of-claims. By prompting the reasoner, the claim filter systematically examines the relationship between each evaluator claim and the original task instruction, discarding any claims lacking a causal linkage to the target probing tasks.

\textbf{Minimizing Overhead.} The generation of claims on the policy agent trajectories and the probed states by the evaluator agents requires processing multiple screenshots and abundant screen descriptions. To minimize the processing overhead, we further filter out noisy states in the trajectories without harming the quality. Specifically, the states related to operations on home-screen and consecutive identical states (some actions do not lead to state transition) are excluded. When there are loops detected in the trajectories, the loops could be discarded if identified by the $\mathcal{J}$ to be unrelated with the task goal $G$ (e.g., some redundant explorations) only using HTML descriptions. 

% \textbf{Task-Specific Observation.} We notice that the reward model might overlook the details in the observed state due to the large amount of complex and noisy information embedded in the screenshots. Therefore, in \sysname, task-specific observations are extracted based on the screenshots and accessibility tree of each state. Specifically, one LLM, serve as a coach to monitor each step of the policy agent and then generate task-specific observation $s^T=\mathrm{LLM}(s$, including the key UI elements and missing expected content. The state $s^T$ is expressed in a JSON format and then translated into text for the final judgement.

% % \dgl{Introduce the way we extract subjective information from each page using a11y tree and screenshot. Discuss the benefits compared with history-only + last several screen-shot. It is even better compared with using all the screenshots as we 1) assign page-wise analysis and 2) do information filtering}
% \subsection{Reward Generation with Generalist Reward Model}

% \dgl{we have two ways:}

% \dgl{1) LLM-As-A-Judge with evidences from actors}

% \dgl{2) Assertion by comparing evidence and expectations}

% \dgl{Describe the guidelines and the prompt content.}

\subsection{Test-time Scaling for States Probing in \sysname} \label{sec: tts}
In some complex tasks or scenarios that a single state probing trial is insufficient for identifying the key evidence, two forms of test-time scaling, including parallel state probing and iterative state probing, are further incorporated to improve the state probing quality in \sysname shown in Figure~\ref{fig:tts}. 

\textbf{Parallel State Probing.} After the policy agent completes the trajectory for one task, the final state is distributed to multiple emulator instances. On each instance, the proactive state probing are conducted in parallel. Later, based on the claims from each evaluator agent and the policy agent, the LLm reasoner formulate the chain-of-claims and assigns the outcome reward. To support the parallel state probing, we record the actions per steps of policy agents and re-execute those actions in sequence on different emulator instances. The target UI elements are matched with the recorded actions via fuzzy matching on key parameters and elements semantics.

\textbf{Iterative State Probing.} The iterative state probing generates new state probing tasks based on the state probing task and claims in the previous round. Specifically, there are $N$ rounds of search. In round $n$, the probing tasks $G_e(n)$ are generated with:
\begin{equation}
G_e(n) \sim \mathcal{J}(G \mid \mathrm{Exp}, \mathcal{E}, L, G_e(n-1), \tau_e), G\in \mathcal{G}, n=1,\dots,N
\end{equation}
The results from different trials are aggregated via majority voting. The quality of the state probing tasks and the quality of collected states can be gradually improved based on the previous experience, which improves the overall reward accuracy.

% \textbf{Efficient Experience Replay.} Instead of running policy agents to execute tasks for each evaluator agent on different emulators, which is neither efficient nor consistent across trajectories due to randomness of LLMs, we propose two kinds of efficient experience replay of policy agents to support the test-time scaling of \sysname. 

% \dgl{Introduce the two ways for test-time scaling. Sequential and parallel evaluation tasks generation and evidence search.}

\section{Evaluation} \label{sec:evaluation}
\subsection{Experiment Settings} \label{sec:setting}

\begin{table}[t]
\centering
% \small
% \setlength{\tabcolsep}{5pt}
\caption{Reward accuracy and F1 across methods and policy agent.}
\begin{tabular}{ll*{3}{cc}cc}
\toprule
\multirow{2}{*}{\textbf{Method}} &
\multirow{2}{*}{\textbf{Last N States}} & \multicolumn{2}{c}{\textbf{V-Droid}} & \multicolumn{2}{c}{\textbf{M3A}} &
\multicolumn{2}{c}{\textbf{UI-TARS-7B}} &
\multicolumn{2}{c}{\textbf{Avg}} \\
\cmidrule(lr){3-8}
% \cmidrule(lr){3-4}\cmidrule(lr){5-6}\cmidrule(lr){7-8}
& & Acc & F1 & Acc & F1 & Acc & F1 & Acc & F1 \\
\midrule
\multirow{2}{*}{DistRL}
& 1    & 71.3 & 76.9 & 70.4 & 68.5 & 84.2 & 25.0 & 85.3 & 56.8 \\
& Full & 87.0 & 88.9 & 81.7 & 79.6 & 89.5 & 14.3 & 86.1 & 60.9 \\
% \midrule
\multirow{2}{*}{DigiRL}
& 1    & 76.5 & 80.3 & 74.8 & 73.4 & 84.2 & 10.0 & 78.5 & 54.6 \\
& Full & 84.5 & 87.1 & 82.6 & 80.8 & 86.8 & 11.7 & 84.6 & 59.9 \\
% \midrule
\multirow{2}{*}{WebRL}        & 1      & 82.6 & 85.1 & 81.7 & 79.2 & 93.0 & 20.0 & 85.8 & 61.4 \\
        & Full      & 85.2 & 87.0 & 81.7 & 79.2 & 93.9 & 22.2 & 86.9 & 62.8 \\
\multirow{2}{*}{Step-Critic}   & 1      & 85.2 & 87.0 & 81.7 & 79.4 & 93.0 & 20.0 & 86.6 & 61.8 \\
  & Full      & 89.6 & 91.0 & 82.6 & 79.6 & 93.0 & 20.0 & 88.4 & 63.6 \\
% \midrule
\sysname & Proactive & \textbf{93.1} & \textbf{93.4} & \textbf{91.4} & \textbf{88.9} &
\textbf{96.5} & \textbf{66.7} & \textbf{93.7} & \textbf{83.0} \\
\bottomrule
\end{tabular}
\label{tab:policy-agent-results}
\vspace{-1em}
\end{table}

\begin{figure}[!t]
    \centering
    \includegraphics[width=1.0\linewidth]{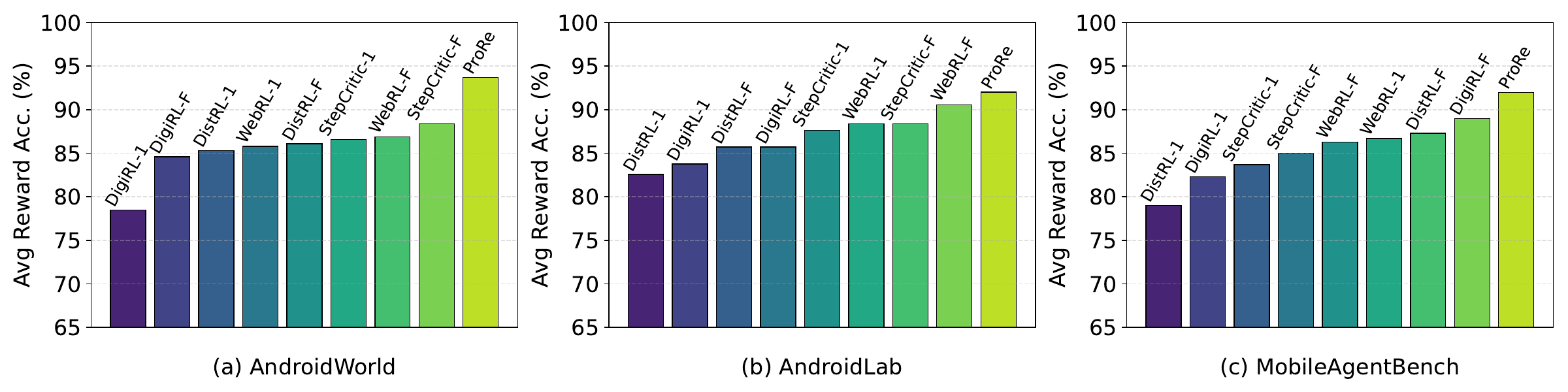}
    \caption{Results Comparison on different benchmarks. The average results on different agents are reported. The \emph{1/F} indicates that the reward uses the last state (\emph{1}) or the full trajectory (\emph{F}).}
    \label{fig:avg_bench}
    \vspace{-1em}
\end{figure}

\textbf{Baselines.}
We compare \sysname with state-of-the-art reward methods, covering both outcome reward models (DigiRL~\cite{bai2024digirl}, DistRL~\cite{wang2024distrl}, WebRL~\cite{qi2024webrl}) and one progress reward model (StepCritic~\cite{lai2025androidgen}). To ensure fairness, we rigorously follow the experimental settings and prompts described in the original papers when reproducing their methods and reporting results. Unless otherwise specified, the results of \sysname are reported \emph{without} test-time scaling when compared against baselines for fairness.

\textbf{Benchmarks.} We conduct comprehensive evaluation of \sysname on over 3k traces collected from three dynamic benchmarks, including AndroidWorld \cite{rawles2024androidworld}, AndroidLab \cite{xu2024androidlab}, and MobileAgentBench \cite{wang2024mobileagentbench}, using state-of-the-art GUI Agents \cite{dai2025advancing, qin2025ui, rawles2024androidworld}. For UI-TARS \cite{qin2025ui}, we adopt UI-TARS-1.5-7B in the naive agentic mode to generate thinking and grounding.

\textbf{Metrics.}
To evaluate the effectiveness of rewards, we report both the \emph{reward accuracy} and \emph{F1 score} by comparing the predicted rewards from baselines and \sysname with the ground-truth rewards provided by the benchmarks. In addition, we measure the \emph{success rate} of policy agents under test-time scaling when guided by these rewards (See \S~\ref{sec:problem_formulation} and Appendix \ref{appendix:tts}).

\textbf{Implementation details.}
We adopt \textit{Gemini-2.5-Pro} as the general-purpose LLM for scheduling state probing tasks and assigning outcome rewards. Unless otherwise specified, V-Droid is employed as the evaluator agent due to its high decision-making quality and prompt execution speed. The step budget for key evidence retrieval is set to be no greater than the length of the policy trajectories.

\subsection{Results Comparison}
    % \dgl{we include qualitative + quantitative results on different benchmarks, policy agents, and evaluator agents.}
% \begin{wrapfigure}{r}{0.3\linewidth}
%     \centering
%     \includegraphics[width=1.0\linewidth]{Figures/Agent-Ability-Reward-Acc.png}
%     \caption{The reward accuracy on agents with different success rate.}
%     \label{fig:agent-ability-rewawrd-acc}
% \end{wrapfigure}

\textbf{Different GUI Agents.} Table~\ref{tab:policy-agent-results} reports the performance of \sysname compared with state-of-the-art baselines. \sysname achieves an average accuracy of 93.7\%, which is 5.3\% higher than the best-performing baselines. Moreover, its F1 score is 19.4\% higher than those of the baselines, demonstrating its robustness in handling diverse mobile GUI agents. We observe that while baselines achieve relatively high accuracy on UI-TARS trajectories, their F1 scores remain low. This discrepancy arises from their inability to correctly judge the success of UI-TARS-1.5-7B trajectories, whose naive agentic mode yields only a 7.9\% success rate. In contrast, \sysname effectively identifies the correct key states through evaluator agents, leading to superior performance on challenging and imbalanced trajectories.

\textbf{Different Benchmarks and Tasks.} As shown in Figure~\ref{fig:avg_bench}, \sysname achieves accuracy improvements of 5.2\%, 1.5\%, and 3.0\% on AndroidWorld, AndroidLab, and MobileAgentBench, respectively. In terms of F1 score, \sysname outperforms the best baseline by 19.4\%, 10.5\%, and 7.5\% on the three benchmarks. These results highlight the robustness of \sysname across diverse applications and task types. While policy agents often struggle to generalize to unseen tasks or applications, \sysname benefits from the execution–probing gap (see \S~\ref{sec:probing}), which makes generalization more attainable. 

\textbf{Extension to PC/Web.} Owing to the decoupled reasoner-actor reward paradigm, \sysname exhibits significant potential for adaptation across diverse environments and tasks, including both PC and web domains. We have further conducted pilot experiments on OSWorld. Specifically, \sysname is evaluated on 100 randomly sampled tasks from OSWorld-PC and all 46 web-based tasks from OSWorld-Chrome. For evaluation, we employ Claude-Sonnet-4.5 as the evaluator agent to perform proactive state probing, and all other settings are consistent with those outlined in \S~\ref{sec:setting}. 

\begin{table}[th]
\centering
\caption{Reward Accuracy on PC and Web tasks.}
\begin{tabular}{lccccc}
\toprule
\textbf{Benchmark} & \textbf{WebRL} & \textbf{DigiRL} & \textbf{DistRL} & \textbf{StepCritic} & \textbf{\sysname} \\
\midrule
OSWorld            & 86.0 & 88.0 & 88.0 & 81.0 & \textbf{92.0} \\
OSWorld-Chrome  & 87.0 & 84.8 & 82.6 & 87.0 & \textbf{93.5} \\
\bottomrule
\end{tabular}
\label{tab:reward_accuracy_pc_web}
\end{table}

As shown in Table \ref{tab:reward_accuracy_pc_web}, across both PC and Web tasks, \sysname achieves the highest reward accuracy, surpassing prior methods by 4.0\% on OSWorld and 6.5\% on OSWorld-Chrome. Existing approaches perform sub-optimally primarily due to incomplete observations of PC/web states and the domain knowledge gap of the reasoners when used as reward models. In contrast, \sysname (i) proactively collects key states/observations by interacting with the PC or website, and (ii) decouples general reasoning from domain-specific GUI judgments through its reasoner–actor paradigm. These results underscore \sysname’s robustness and generalization capability across different platforms and task.

\begin{wrapfigure}{r}{0.7\linewidth}
    \centering
    \includegraphics[width=1.0\linewidth]{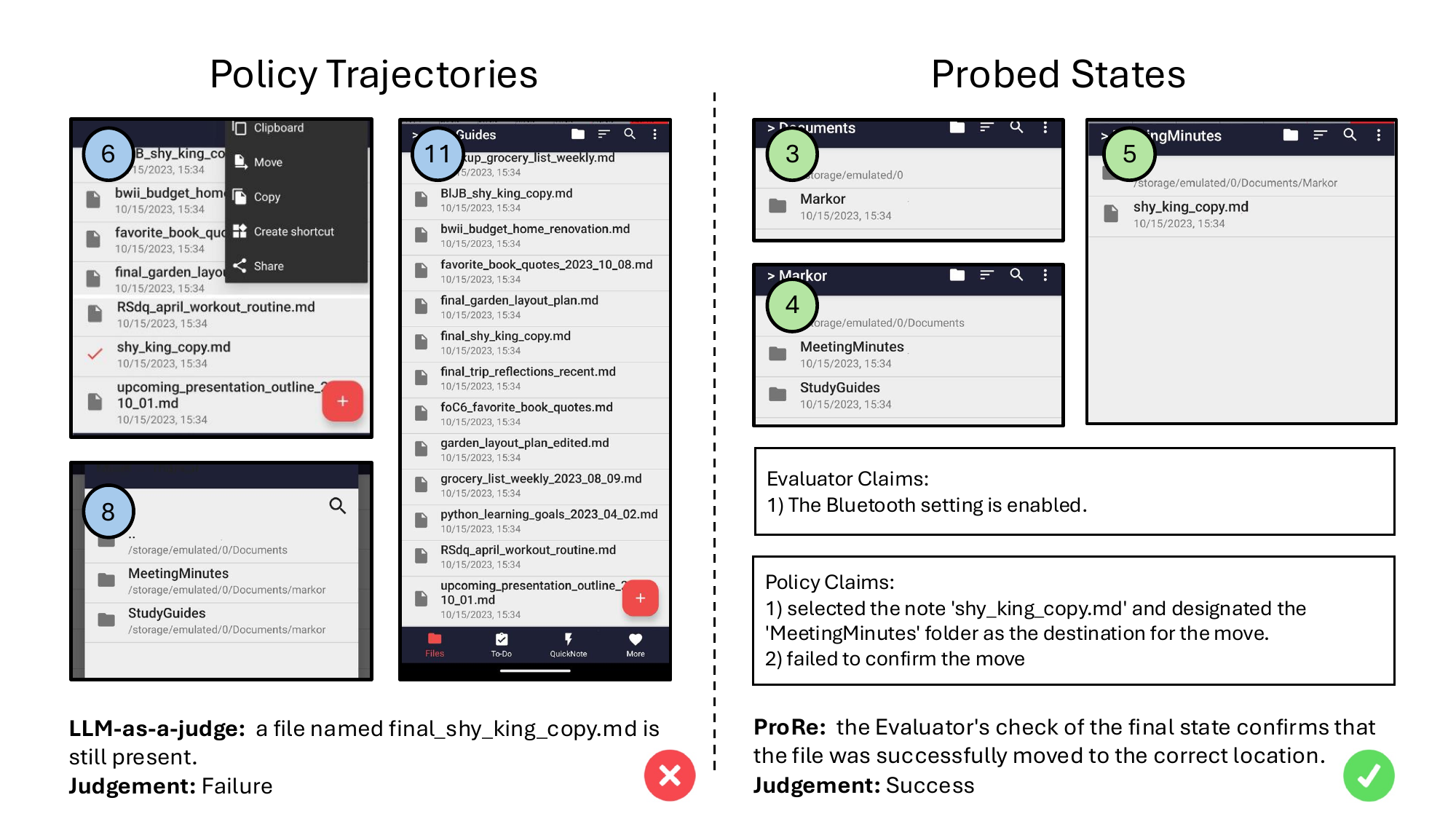}
    \caption{One quantitative example. The task is "\textit{Move the note shy\_king\_copy.md from StudyGuides to MeetingMinutes.}".}
    \label{fig:quantitative} 
    % \vspace{-1.5em}
\end{wrapfigure}

% \textbf{Illustrative Examples.} As shown in Figure~\ref{fig:quantitative}, the policy agent finds the targert file, perform the moving actions, and return to the StudyGuides folder. However, the LLM-as-a-judge is misleaded by the final\_shy\_king\_copy.md due to the excessive information on the page and makes the wrong judgement. In contrast, \sysname probed the states in the target folder with the evaluator agent, which proves the success of the policy agents. This example also present the execution-probing gap: the whole trajectories is more than 10 steps whereas the evidence retrieval only takes 5 steps.

\begin{figure*}[!t]
  \centering
  % Left minipage: subfigures (a) and (b)
  \begin{minipage}[t]{0.66\textwidth}
    \centering
    \includegraphics[width=0.9\textwidth]{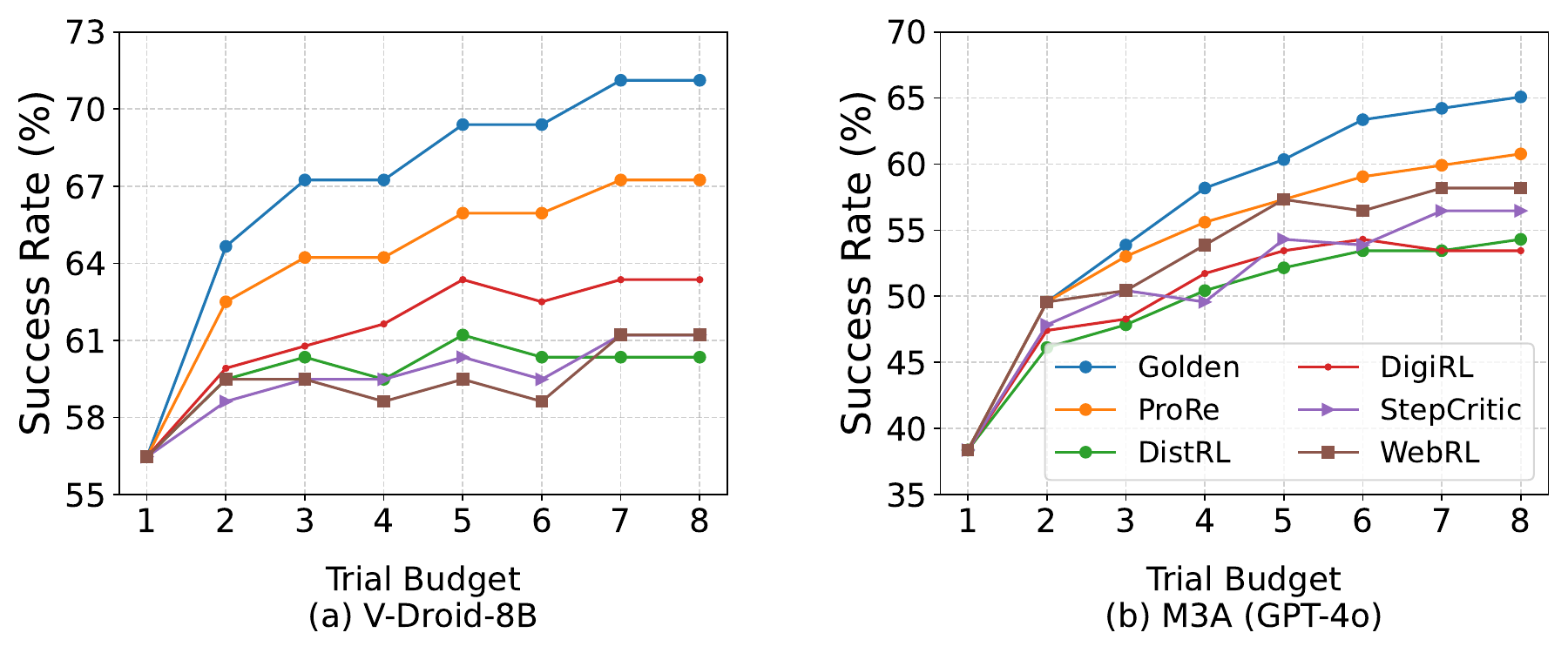}
    \caption{Test-time scaling of policy agents with different rewards. 
  (a) V-Droid-8B, (b) M3A (GPT-4o).}
    \label{fig:test-time-scaling-aw}
  \end{minipage}%
  \hfill
  % Right minipage: independent figure
  \begin{minipage}[t]{0.32\textwidth}
    \centering
    \includegraphics[width=\textwidth]{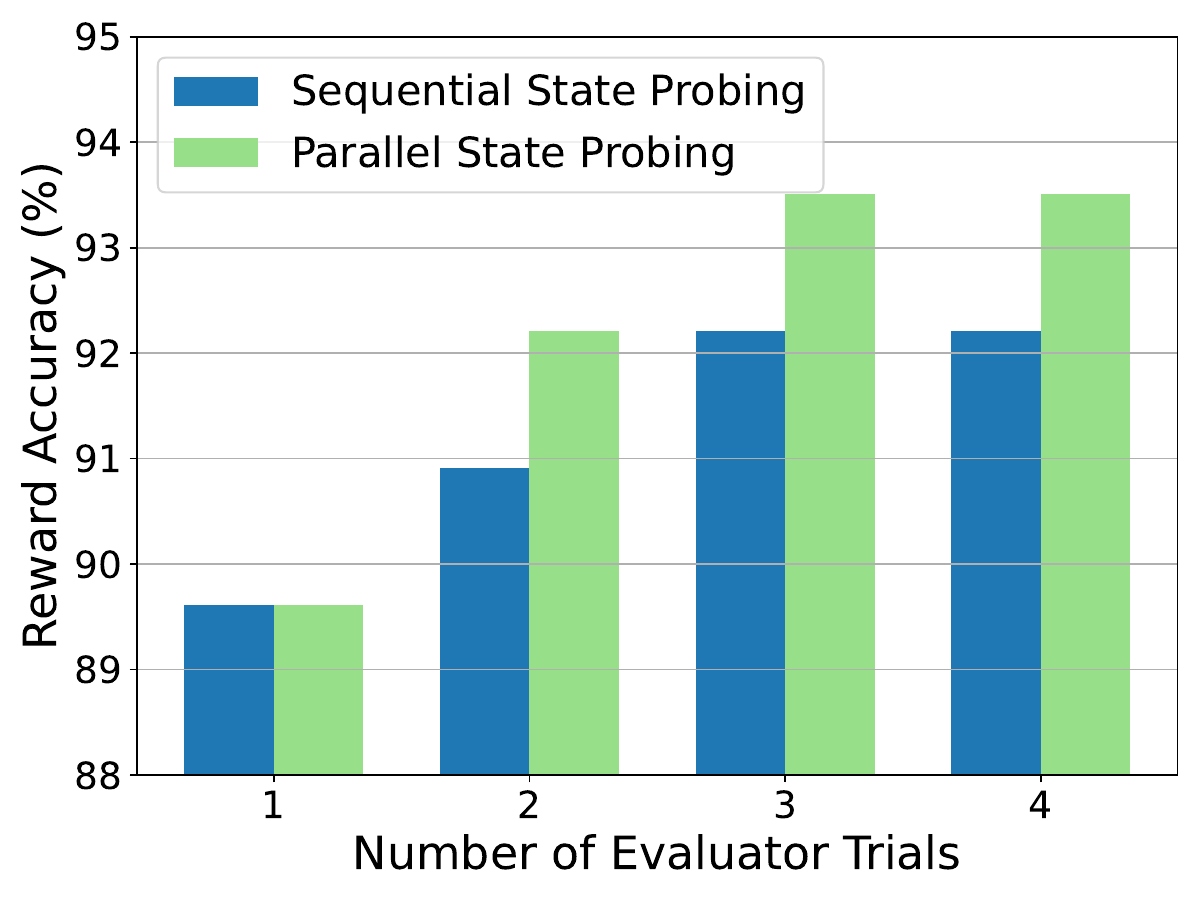}
    \caption{Test-time scaling of \sysname on challenging tasks.}
    \label{fig:tts_prore}
  \end{minipage}
  % \vspace{-0.5em}
\end{figure*}

\textbf{Test-Time Scaling for Policy Agents.} We further evaluate the success rate (SR) of two policy agents, V-Droid~\cite{dai2025advancing} and M3A (GPT-4o)~\cite{rawles2024androidworld}, under different trial budgets. Figure~\ref{fig:test-time-scaling-aw} shows that, guided by \sysname, the success rate (SR) of V-Droid improves from 56.5\% to 67.2\%. Similarly, the SR of M3A (GPT-4o) increases by 22.4\%. The SR gains achieved with \sysname are 3.9\% and 4.3\% higher than those obtained with other reward methods, demonstrating its superiority in guiding policy rollouts. To further validate this, we conduct large-scale simulations based on Lemma~\ref{lem:reward_success}, which highlight the effectiveness of accurate rewards in enhancing test-time scaling of policy agents (see Appendix~\ref{appendix:tts}).

\textbf{Illustrative Examples.} Figure~\ref{fig:quantitative} shows that the policy agent successfully locates the target file, performs the necessary move actions, and returns to the \texttt{StudyGuides} folder. However, the LLM-as-a-judge is misled by the presence of \texttt{final\_shy\_king\_copy.md} due to excessive clutter on the final screen and consequently makes an incorrect judgment. In contrast, \sysname proactively probes the relevant states within the target folder using an evaluator agent, which provides verifiable evidence of the policy agent’s success. This example also highlights the \textit{execution-probing gap}: while the execution trajectory spans 11 steps, the evaluator only requires 5 steps to probe the key states. More examples are provided in Appendix~\ref{appendix:examples}. 

% Figure \ref{fig:test-time-scaling-aw} shows that guided by \sysname, the SR of V-Droid improves from 54.3\% to 67.2\%. The SR of M3A(GPT-4o) improves by 22.4\%. The SR gain with \sysname are \textcolor{red}{4.3\%} and 4.3\% higher compared with other reward methods. The results show that \sysname can consistently provide better reward to guide the rollouts of policy agents. We further conduct large-scale simulation on policy agents based on Lemma \ref{lem:reward_success} to highlight the effectiveness of rewards on the test-time scaling of policy agents in \S~\ref{appendix:tts}. 

\begin{table}[!t]
\centering
% \small
\caption{Ablation study of design components in \sysname.}
\label{tab:ablation-design}
\scalebox{0.85}{
\begin{tabular}{cccc|rrrrr}
\toprule
\multicolumn{4}{c|}{\textbf{Design Components}} & \multicolumn{5}{c}{\textbf{Metrics}} \\
\cmidrule(lr){1-4}\cmidrule(lr){5-9}
\makecell{\textbf{Proactive}\\\textbf{State Probing}} &
\makecell{\textbf{State Probing}\\\textbf{Task Scheduling}} &
\makecell{\textbf{Chain-of-Chaims}} &
\makecell{\textbf{Iterative State }\\\textbf{Probing}} &
\textbf{Acc} & \textbf{TP} & \textbf{TN} & \textbf{FP} & \textbf{FN} \\
\midrule
\xmark & \xmark & \xmark & \xmark & 88.8 & 49.1 & 39.7  & 7.7 & 3.4 \\
\cmark & \xmark & \xmark & \xmark & 89.5 & 45.6 & 43.6 & 2.6 & 7.9 \\
\cmark & \cmark & \xmark & \xmark & 91.4 & 45.7 & 45.7 & \textbf{1.7} & 6.9 \\
\cmark & \cmark & \cmark & \xmark & 93.1 & 49.1 & 44.0 & 3.4 & 3.4 \\
\cmark & \cmark & \cmark & \cmark & \textbf{94.8} & \textbf{50.0} & \textbf{44.8} & 2.6 & \textbf{2.6}\\
\bottomrule
\end{tabular}
}
% \vspace{-1em}
\end{table}

\subsection{Ablation Study}
% We further show the effectiveness of each design component in \sysname with an ablation study. When without state probing tasks scheduling, a simple rule-based probing task is obtained by putting the original instruction $T$ into "What are the key states to verify whether the task \{$G$\} is completed." When without chain-of-claims, all the policy and evaluators' observations are inputted to the reasoner.

% We further validate the effectiveness of each design component in \sysname through ablation studies. Without probing task scheduling, a simple rule-based probing task is constructed by prompting: "What are the key states to verify whether the task {$G$} is completed?" Without chain-of-claims reasoning, the reasoner receives the raw observations from both the policy and evaluator agents.

We further validate the effectiveness of each design component in \sysname through ablation studies. When probing task scheduling is removed, we replace it with a simple rule-based strategy by prompting: \textit{"What are the key states to verify whether the task \{$G$\} is completed?"} Without chain-of-claims reasoning, the reasoner directly receives the raw observations from both the policy and evaluator agents, without structured analysis.

Table~\ref{tab:ablation-design} demonstrates the contribution of each design component in \sysname. Without explicit guidance from the reasoner, evaluator agents navigate with probing tasks generated with simple rules, which provides marginal improvement. When the reasoner schedules probing tasks for evaluator agents, the accuracy increases substantially to 91.4\%, underscoring the effectiveness of separating reasoning and planning (by the reasoner) from execution (by the evaluators). In addition, incorporating chain-of-claims reasoning further improves accuracy by 1.7\%, highlighting the importance of summarizing low-level GUI details from trajectories and analyzing the relationships between policy and evaluator claims. Finally, iterative state probing in \sysname boosts performance to 94.8\%, as additional probing and refinement yields more complete observations of key states. Besides, without the claim filter, we observe a 1.7\% reward accuracy drop on AndroidWorld benchmark, underscoring the necessity of eliminating irrelevant or misleading claims prior to the chain-of-claims.

Figure~\ref{fig:tts_prore} further illustrates the benefits of parallel and iterative state probing, especially on more challenging tasks. Notably, parallel state probing yields larger performance gains compared with iterative probing. A possible explanation is that the reasoner, lacking domain-specific GUI knowledge, is less effective at leveraging the intermediate observations and action histories provided by the evaluator agents to refine subsequent probing tasks.

\begin{table*}[!t]
\centering
\begin{minipage}[t]{0.42\textwidth}
    \centering
    \caption{Reward accuracy of \sysname with different reasoners.}
    \label{tab:reasoners_comparison}
    \begin{tabular}{lcc}
    \toprule
    \textbf{Reasoners} & \textbf{Acc} & \textbf{F1} \\
    \midrule
    Gemini-2.5-Pro & 93.1 & 93.4 \\
    Gemini-2.5-Flash & 87.7 & 87.7 \\\midrule
    % Gemini-2.0-Flash & 81.6 & 82.9 \\
    GPT-5 & 86.2 & 86.0 \\
    GPT-4o & 85.0 & 86.0 \\
    \bottomrule
    \end{tabular}
\end{minipage}
\hfill
\begin{minipage}[t]{0.54\textwidth}
    \centering
    \caption{Reward accuracy of \sysname with different evaluator agents.}
    \label{tab:different_evaluators_comparison}
    \begin{tabular}{l c c c}
    \toprule
    \textbf{Evaluator Agent} & \textbf{Policy SR} & \textbf{Acc} & \textbf{F1} \\
    \midrule
    V-Droid-8B       & 59.5 & 93.1 & 93.4 \\
    UI-TARS-72B      & 35.7 & 86.2 & 87.3 \\
    Qwen3-VL-4B      & 45.3 & 85.7 & 86.0 
    \\ \midrule
    M3A (GPT-5)      & 56.9 & 90.5 & 91.7 \\
    M3A (GPT-4o)     & 41.3 & 88.3 & 87.4 \\
    \bottomrule
    \end{tabular}
\end{minipage}
\vspace{-1.5em}
\end{table*}

\textbf{Different Reasoners.} 
We further vary the reasoners in \sysname. Table~\ref{tab:reasoners_comparison} shows that reasoners equipped with built-in chain-of-thought capabilities are more effective at analyzing the relationships between policy and evaluator claims, leading to higher reward accuracy and F1 scores.

% \dgl{move this to the motivation section? But we have the score on the llm. This seems to be controdictory to the insight we want to deliver.}

\textbf{Different Evaluator Agents.}
We further investigate the impact of evaluator agent capability. \emph{Policy SR} denotes the task success rate of an agent when deployed as a policy. As shown in Table~\ref{tab:different_evaluators_comparison}, stronger evaluator agents are able to probe key states from the environment more effectively, thereby achieving higher reward accuracy. Moreover, fine-tuned small GUI agents can outperform large generalist agentic systems when used as evaluators, owing to their domain-specific GUI knowledge. We also notice that Qwen3-VL-4B yields notably lower reward accuracy compared with larger models. We hypothesize that this drop stems from the smaller models’ weaker ability to generalize to unseen probing tasks and identify key observations even with high-quality probing instructions. 

% \begin{table}[!t]
% \centering
% \caption{Reward accuracy of \sysname with different evaluator agents.}
% \label{tab:different_evaluators_comparison}
% % \renewcommand{\arraystretch}{1.2}
% \begin{tabular}{l l c c c c c c}
% \toprule
% \textbf{Evaluator Type} & \textbf{Evaluator Agent} & \textbf{Policy SR} & \textbf{Acc} & \textbf{TP} & \textbf{TN} & \textbf{FP} & \textbf{FN} \\
% \midrule
% \multirow{2}{*}{Agentic Models}
%     & V-Droid-8B & 59.5 & 92.1 & 47.4 & 44.7 & 1.8 & 6.1 \\
%     & UI-TARS-72B &   &   &  &   &   &   \\ \midrule
% \multirow{2}{*}{Agentic Systems}    & M3A(GPT-5)     & 59.5 &  87.9 & 55.2 & 32.8 & 7.8 & 4.3 \\
%     & M3A(GPT-4o) & 41.3 & 88.3 & 49.5 & 39.0 & 6.5 & 5.2 \\

% % \multirow{2}{*}{AndroidLab \dgl{tbd}}
% %     & V-Droid &  & &  & &  & \\
% %     & M3A(GPT-4o) & & &  & &  & \\
% % \midrule
% % \multirow{2}{*}{MobileAgentBench}
% %     & V-Droid & 49.0 & 85.0 & 40.0 & 45.0 & 10.0 & 5.0 \\
% %     & M3A(GPT-4o)     & 63.0 & 91.0 & 40.0 & 51.0 & 5.0  & 4.0 \\
% \bottomrule
% \end{tabular}
% \end{table}

% \subsection{Simulation Analysis for Reward System Guided Test-Time Scaling}

% \dgl{quality analysis on evaluation tasks?}

% \dgl{success rate on the execution of evaluation tasks? The link between evidence obtained and the correct judgement?}

\textbf{Overhead Analysis.} The total cost of \sysname is approximately \$0.06 per agent task. Among the components, state probing task scheduling and chain-of-claims contribute about \$0.013 and \$0.050, respectively. Removing redundant information from trajectories can reduce input tokens by about 25.9\% without degrading performance. Overall, \sysname remains significantly more cost-efficient and scalable compared to hiring human annotators. We further provide a detailed per-task cost comparison and long-term cost estimation between \sysname and the baselines in Appendix \ref{appendix:cost_comparison}.

\section{Discussions}
\textbf{Online RL with \sysname}. Prior work has shown that online reinforcement learning (RL) can achieve substantially better performance when guided by accurate reward signals \cite{qi2024webrl, wang2024distrl}. After policy agents execute actions in an online RL setting, \sysname can be seamlessly integrated to provide more precise reward assignments with only moderate overhead. Nevertheless, we defer a full exploration of online RL with \sysname to future work.

% Online RL is widely observed to be more effective with better rewards \cite{qi2024webrl, wang2024distrl}. After the execution of policy agents during online RL, \sysname could take over to assign significant better reward to the agent with acceptable additional overhead. We leave the Online RL as future work due to limited access to GPU and emulator clusters.

\textbf{Co-evolution of Policy and Evaluator Agents.} In \sysname, the policy agent and the evaluator agent can be instantiated from the same underlying model, creating a unique opportunity for co-evolution. Stronger evaluator agents enhance reward accuracy, which in turn improves the policy agent’s success rate. As the policy agent becomes stronger through test-time scaling or training, it enables the evaluator to achieve higher success on state probing tasks and further improve reward accuracy. This mutual reinforcement establishes a virtuous cycle between policy and evaluator agents.

% \textbf{Co-evolving between Policy and Evaluator Agents.} In \sysname, the evaluator agent and the policy agent can be instantiated as the same underlying model, enabling a unique co-evolution opportunity. A stronger evaluator agent leads to better reward accuracy and better success rate of policy agent. As the policy agent improves through test-time scaling or training, the evaluator agents’ success rate on state probing, and thus the reward accuracy, also improves over time. 

% \textbf{Applicability to Other Domains.} \sysname leverages the insight that state probing tasks are often simpler than the original tasks that require multi-step precise executions. Although our current evaluation focuses on mobile devices, the framework is readily extensible to other domains, such as computer use, web use, and embodied intelligence, where we plan to explore in the future.

% \textbf{Applicability to Other Domains.} \sysname is built on the oppournities that state probing tasks could be easier compared with the original tasks involving multi-steps exact executions. We argue that while \sysname is currently implemented and tested on mobilephone, we argue that it can be easily extended to other scenarios, such as computer-use, web-use, or embodied intelligence. In those cases, the reward system can scale first with \sysname to guide the improvement of policy agents.
\section{Conclusions}
\label{sec:conclusions}
Unlike existing trajectory-based LLM-as-a-Judge approaches, \sysname introduces a proactive reward system for GUI agents that integrates a general-purpose reasoner with domain-specific evaluator agents. The evaluator agents proactively probe key states based on probing tasks scheduled by the reasoner, while the reasoner make final judgments based on the chain-of-claims from the evaluator agents. Extensive experiments across diverse tasks, applications, and agents demonstrate the effectiveness of \sysname, as well as its effectiveness in guiding the test-time scaling of policy agents.

% \begin{acks}
\section*{Acknowledgements}
This research is partially supported by Singapore Ministry of Education under its AcRF Tier 1 grant RT14/22, the Global STEM Professorship Scheme of Hong Kong, the HKUST start up grant, and the Research Grants Council (RGC) General Research Fund (GRF) 16210425.
% \end{acks}

\bibliography{iclr2026_conference}
\bibliographystyle{iclr2026_conference}

\appendix
\section{The Use of Large Language Models (LLMs)}
This work studies rewarding LLM-based GUI agents with a proactive reward system. LLMs were involved in three aspects of our research: (i) serving as the backbone for the GUI agent and baseline implementations, (ii) supporting framework design, and (iii) assisting with polishing the writing of the manuscript.  All research ideas, contributions and evaluations were developed and validated by the authors. No LLM is considered an author.  

\setcounter{lemma}{0}
\section{Test-time Scaling for Policy Agents} \label{appendix:tts}
\subsection{Proof of Lemma~\ref{lem:reward_success}}
% We define the test-time scaling of policy agent as: the policy agent rolled out for a task with trial budge N. The reward methods would evaluate the trial after the policy finishes. Once the reward methods determines the policy agent success or the policy agent meet the trial budget, the policy agent stops and submit the corresponding trial for the final check. 

We define \emph{test-time scaling} as the procedure where a policy agent is rolled out on a given task with a maximum trial budget $N$. After each trial, a reward method evaluates the trajectory produced by the policy agent. If the reward method determines the trajectory to be successful, the process terminates early and the corresponding trial is submitted as the final output. Otherwise, the policy agent continues to the next trial, until either a successful trajectory is identified or the trial budget $N$ is exhausted. If no positive judgment is given within $N$ trials, the final trial is submitted as the output.

\begin{lemma}[Restated]
Let the success rate of the policy agent be $p_a$ and the reward accuracy be $p_c$. 
Then, under test-time scaling with trial budget $N$, the final success rate $P_{final}$ satisfies
\[
P_{final} = \frac{p_ap_c}{q}\big[1-(1-q)^N\big] + p_a(1-q)^N, 
\quad \text{where } q = p_ap_c+(1-p_a)(1-p_c).
\]
In particular, given $p_a > 0$, $P_{final}$ monotonically increases with respect to $p_c$ whenever $p_c > 0.5$.
\end{lemma}

% Run $N$ trials and stop at the first positively judged trial if one exists; otherwise (i.e., if all $N$ judgments are non-positive), return the last trial.

\begin{proof}
In each trial, the policy agent succeeds with probability $p_a$. The reward model outputs a positive judgment with probability
\[
q \;=\; \Pr(\text{TP})+\Pr(\text{FP}) \;=\; p_ap_c + (1-p_a)(1-p_c).
\]
where TP and FP refers to true positives and false positives. 

\emph{(i) Formula.} The probability that at least one positive reward appears is $1-(1-q)^N$. Given a positive reward, the probability that the trajectory corresponds to a truly successful one is
\[
\Pr(\text{success} \mid \text{positive}) \;=\; \frac{\Pr(\text{TP})}{\Pr(\text{positive})} \;=\; \frac{p_a p_c}{q}.
\]
Hence when there is at least one positive reward in the N trials, the success probability is $\frac{p_a p_c}{q}\,[1-(1-q)^N]$.

If no positive reward appears, the last submitted trial is an unfiltered trial whose success probability is $p_a$.
Summing the two cases yields
\[
P_{final} \;=\; \frac{p_ap_c}{q}\bigl[1-(1-q)^N\bigr] + p_a(1-q)^N.
\]

\emph{(ii) Monotonicity in $p_c$ for $p_c>\tfrac12$.}
When $p_c>\tfrac12$, a positive reward is more likely to be a true success than a failure.
Increasing $p_c$ simultaneously increases the true positive rate among judged positives and decreases false positives, thereby making the first positive more likely to be a true success. Formally, differentiating $P_{final}$ with respect to $p_c$ shows $\partial P_{final}/\partial p_c>0$ whenever $p_a>0$ and $p_c>\tfrac12$ (details omitted for brevity). 
Thus $P_{final}$ monotonically increases in $p_c$ on $(\tfrac12,1]$.
\end{proof}

\begin{figure}[!t]
  \centering
  \begin{subfigure}[b]{0.30\textwidth}
    \centering
    \includegraphics[width=\textwidth]{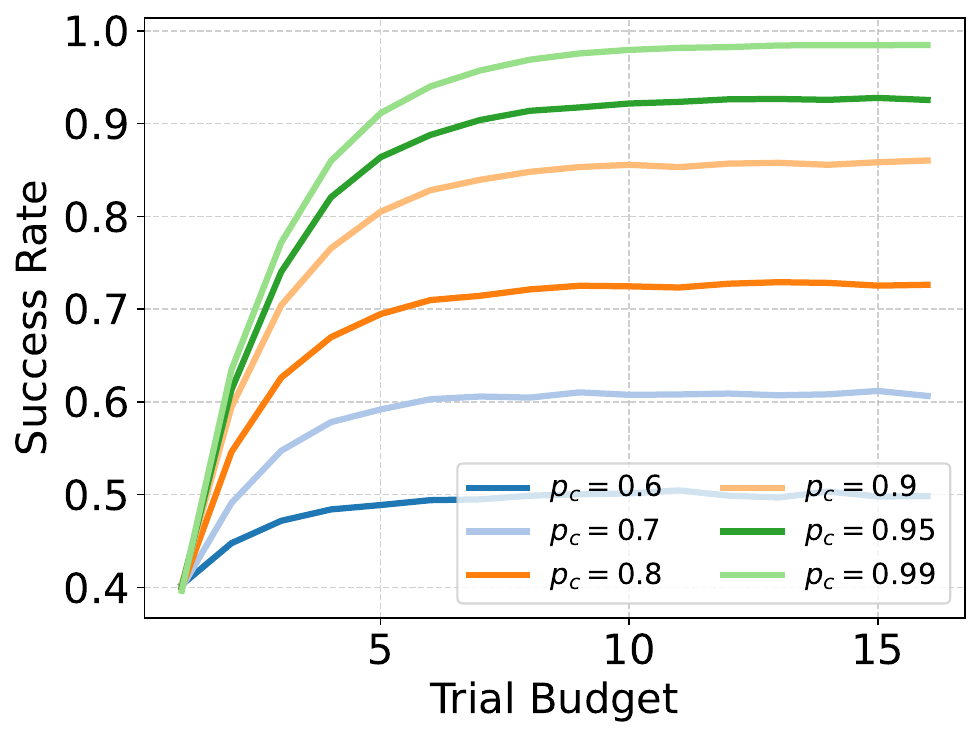}
    \caption{}
    \label{fig:tts_simu}
  \end{subfigure}
  % \hfill
  \hspace{1em}
  \begin{subfigure}[b]{0.30\textwidth}
    \centering
    \includegraphics[width=\textwidth]{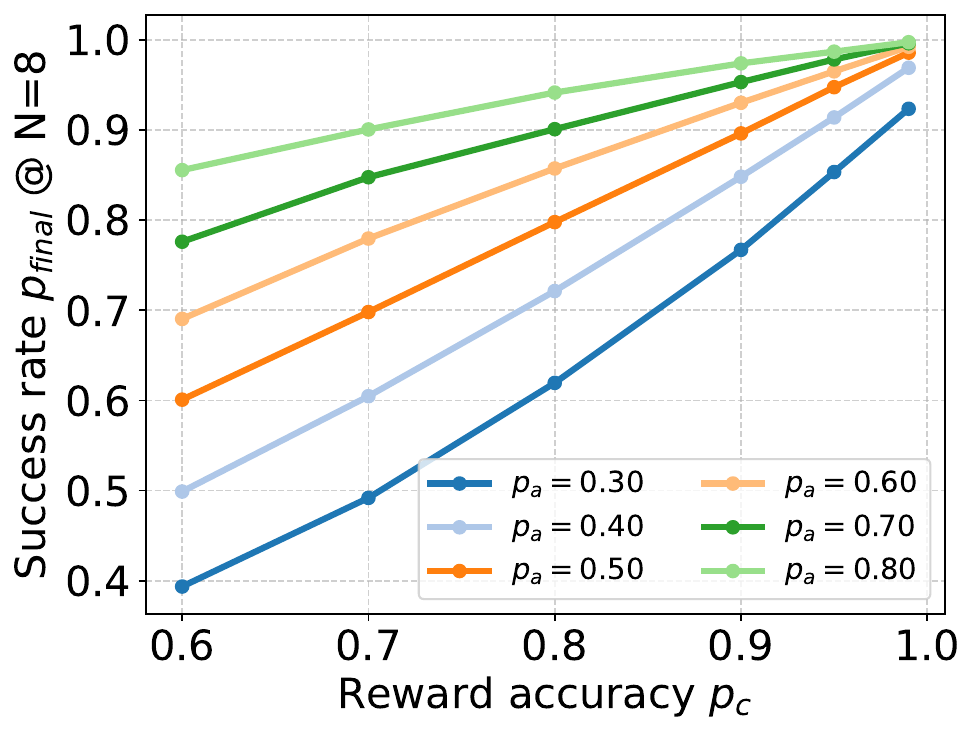}
    \caption{}
    \label{fig:tts_final}
  \end{subfigure}
  % \hfill
  \hspace{1em}
  \begin{subfigure}[b]{0.30\textwidth}
    \centering
    \includegraphics[width=\textwidth]{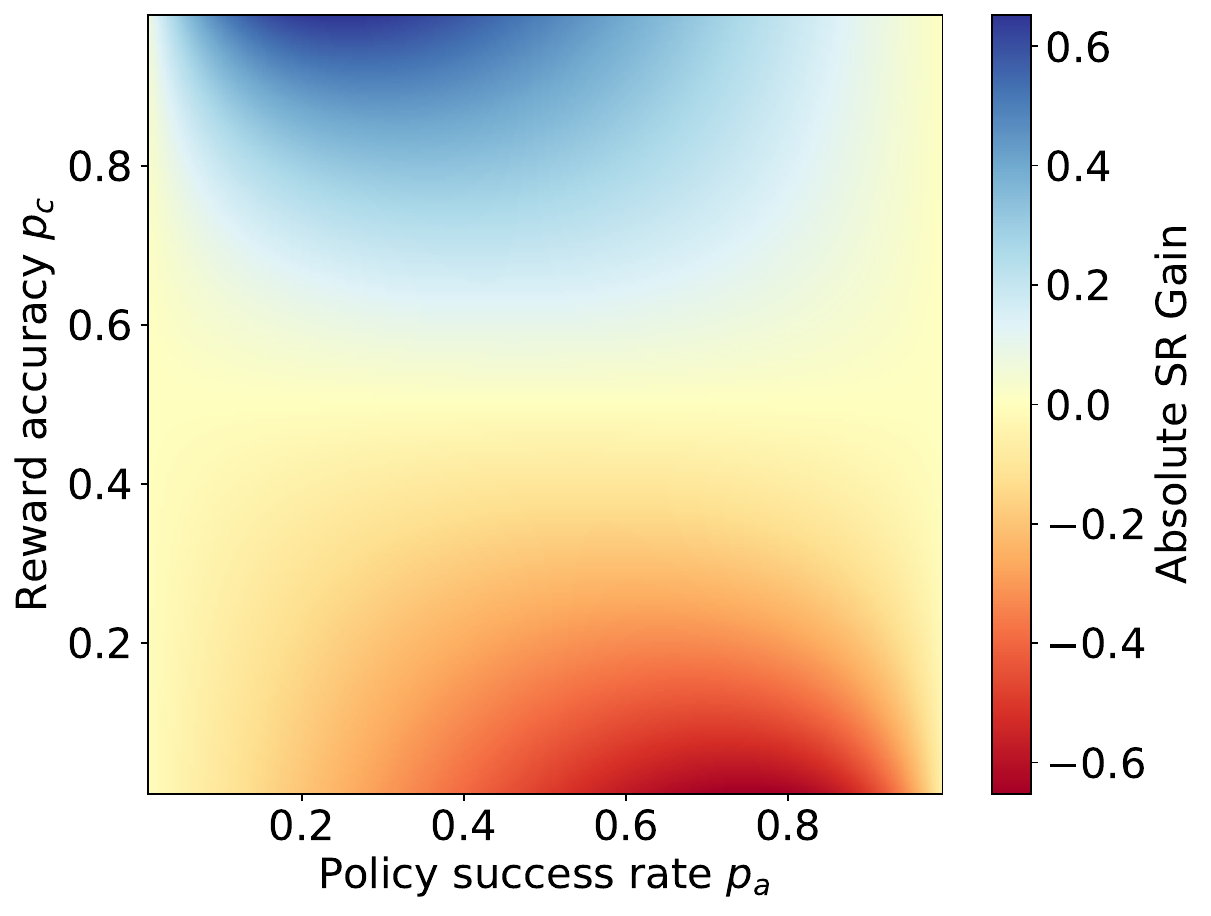}
    \caption{}
    \label{fig:tts_gain}
  \end{subfigure}

  \caption{
  Simulation for test-time scaling of Policy Agents. (a) The success rate of a policy agent increases steadily as the trial budget grows. (b) The final success rate $p_{final}$ monotonically improves with higher reward accuracy across different policy agents, highlighting the importance of reliable evaluation. (c) Reward accuracy contributes the largest absolute success rate (SR) gain for policy agents with moderate baseline SR (0.2–0.6).}
  \label{fig:tts simu}
\end{figure}

\textbf{Implications of Test-time Scaling of GUI Agents.} 
Test-time scaling of GUI agents is closely connected to both the exploration capabilities of GUI agents on applications and tasks, as well as their training efficiency. 
On complex out-of-domain tasks, a GUI agent may actively explore applications, accumulate experience, and iteratively refine its trajectories to achieve the task instructions. 
During training, more effective rollouts enabled by test-time scaling can generate larger-scale datasets with higher quality, while keeping the overall budget fixed.

\subsection{Large-scale Simulation Analysis.} To further study the impact of the reward accuracy on the performance of policy agent during test-time scaling, we conduct large scale simulation for different $P_a$ and $P_c$ based on Lemma \ref{lem:reward_success}. The simulation is repeated for 50K times and the average results are reported in Figure. \ref{fig:tts simu}. 

As shown in Figure~\ref{fig:tts simu}, when the reward accuracy is greater than 50\%, a higher reward accuracy consistently leads to a higher success rate of the policy agent, since additional rollouts increase the likelihood of discovering a successful trajectory. We also observe that in most cases $p_{final}$, indicating the upper capability boundary of the policy agent and underscoring the importance of continuously improving the decision-making ability of GUI agents. The results on realistic tasks and applications in Figure~\ref{fig:test-time-scaling-aw} exhibit a similar trend, further highlighting the importance of reward quality in boosting performance.

\section{State Probing Tasks Examples} \label{appendix:probing_tasks_example}
We further provide illustrative examples of the scheduled state probing tasks in Table~\ref{tab:probing_tasks_example}. Compared to the original tasks, these probing tasks are generally easier (shown in Table~\ref{tab:dualob_subcols}), as they only require the evaluator agent to navigate to the relevant UI page without performing content editing or modification. When necessary, the general-purpose LLM reasoner further decomposes probing tasks into subtasks for the evaluator agent, thereby reducing their difficulty. Moreover, the evaluator agents execute based on the prior interactions of the policy agent, which helps to further mitigate navigation challenges.

\begin{longtable}{p{0.47\linewidth} p{0.47\linewidth}}
\caption{Examples for the original tasks and the corresponding probing tasks} \label{tab:probing_tasks_example} \\
\toprule
\textbf{Original Task} & \textbf{Generated Probing Tasks} \\
\midrule
\endfirsthead
Open the file task.html in Downloads in the file manager; when prompted open it with Chrome. Then click the button 5 times, remember the numbers displayed, and enter their product in the form. & What is the value in the input field on the task.html page in Chrome? \\
\midrule
Take one photo. & Find the most recently taken photo in the gallery. \\
\midrule
Create a timer with 0 hours, 16 minutes, and 35 seconds. Do not start the timer. & Confirm the timer is set to 16 minutes and 35 seconds and is not running. \\
\midrule
Create a new contact for Hugo Pereira. Their number is +13920741751. & What is the phone number for the contact Hugo Pereira? \\
\midrule
Add the expenses from expenses.jpg in Simple Gallery Pro to pro expense. & Show the expenses from expenses.jpg in the pro expense app. \\
\midrule
Go through the transactions in my\_expenses.txt in Markor. Log the reimbursable transactions in the pro expense. & What are the logged transactions in the pro expense file in Markor? \\
\midrule
Delete all but one of any expenses in pro expense that are exact duplicates, ensuring at least one instance of each unique expense remains. & Verify the pro expense list contains no duplicate entries. \\
\midrule
Delete the following expenses from pro expense: Rental Income. & Find the ``Rental Income'' expense in the pro expense app. \\
\midrule
Delete the file q2a8\_fancy\_banana.mp3 from the Android filesystem located in the Notifications folder within the sdk\_gphone\_x86\_64 storage area. & Check the Notifications folder for the file q2a8\_fancy\_banana.mp3. \\
\midrule
Move the file holiday\_photos.jpg from Podcasts within the sdk\_gphone\_x86\_64 storage area to the DCIM within the same sdk\_gphone\_x86\_64 storage area in the Android filesystem. & Check if holiday\_photos.jpg is in the DCIM folder and not in the Podcasts folder. \\
\midrule
Update the Markor note 2023\_08\_10\_neat\_wolf.txt by adding the following text, along with a new blank line before the existing content: ``RnI8sP34yDzJQbvkfplR'', and rename it to busy\_wolf\_2023\_07\_23.txt. & In Markor, open the note busy\_wolf\_2023\_07\_23.txt and show its content. \\
\midrule
Create a new note in Markor named 2023\_01\_26\_wise\_yacht.md with the following text: Ignorance is bliss. & In Markor, what is the content of the note 2023\_01\_26\_wise\_yacht.md? \\
\midrule
Merge the contents of Markor notes tough\_frog\_2023\_08\_05.txt, proud\_cat\_edited.txt and 2023\_08\_21\_friendly\_koala.md (in the same order) into a new Markor note named mIObBbo4 and save it. Add a new line between the content of each note. & What are the contents of the Markor note mIObBbo4? \\
\midrule
In Markor, move the note shy\_king\_copy.md from StudyGuides to MeetingMinutes. & Find the note shy\_king\_copy.md in the MeetingMinutes folder. \\
\midrule
Is the note titled `To-Do List' in the Joplin app marked as a todo item? Respond with either 'True' if it is a todo or 'False' if not. & Check the to-do status of the `To-Do List' note in Joplin. \\
\midrule
What quantity of spirulina do I need for the recipe `Chicken Alfredo' in the Joplin app? Express your answer in the format amount unit where both the amount and unit exactly match the format in the recipe. & What is the quantity of spirulina in the Joplin recipe `Chicken Alfredo'? \\
\midrule
Open the contacts app. Clear any pop-ups that may appear by granting all permissions that are required. & Verify the contacts app is open and no permission pop-ups are visible. \\
\midrule
Add a favorite location marker for 47.1303814, 9.5930117 in the OsmAnd maps app. & Find the favorite location marker for 47.1303814, 9.5930117 in My Places. \\
\midrule
Add a location marker for Planken, Liechtenstein in the OsmAnd maps app. & Find the map marker for Planken, Liechtenstein. \\
\midrule
Add the recipes from recipes.jpg in Simple Gallery Pro to the Broccoli recipe app. & Confirm recipes from recipes.jpg are in the Broccoli app. \\
\bottomrule
\end{longtable}

\section{Additional Examples}
\label{appendix:examples}
We further include additional illustrative examples in Figure~\ref{fig:add_example1}, Figure~\ref{fig:add_example2}, Figure~\ref{fig:add_example3}, and Figure~\ref{fig:add_example4} to demonstrate the limitations of LLM-as-a-judge and the effectiveness of \sysname.

The primary limitations of trajectory-based LLM-as-a-judge approaches for GUI agents are: i) Incomplete state observations of the environment, which hinder accurate reasoning and judgment; and ii) Lack of GUI domain expertise, making it difficult for LLMs to interpret complex UI-related details and the UI logic.

% \textbf{Incomplete State Observations.} Figure~\ref{fig:add_example1}, Figure~\ref{fig:add_example3} and Figure~\ref{fig:add_example4} shows that the policy agents only observe partially to the limited elements presented on screen, which miss many important contents that can help decides the success or failure of the policy agent. For example, in Figure~\ref{fig:add_example3}, the policy observations only include a part of the event names, which leads to incorrect answers from the policy agents. 

\textbf{Incomplete State Observations.} Figure~\ref{fig:add_example1}, Figure~\ref{fig:add_example3}, and Figure~\ref{fig:add_example4} illustrate that the policy agents only observe a partial view of the environment due to the limited UI elements visible on screen and the APPs design. As a result, the trajectories miss critical information necessary for determining task success or failure. For example, in Figure~\ref{fig:add_example3}, the policy agent's observation includes only part of the event names, which leads to an incorrect answer.

\textbf{Lack of GUI Domain Expertise.} Figure~\ref{fig:add_example2} shows that the policy agent chooses to turn on Bluetooth by clicking the \texttt{Pair new device} button. However, the reasoner lacks the GUI-specific knowledge to recognize that this action implicitly triggers the activation of Bluetooth, and therefore fails to make the correct judgment.

% \textbf{Lack of GUI Domain Expertise.} Figure~\ref{fig:add_example2} shows that the policy agent choose to turn bluetooth on via clicking the button "Pair new device". However, the reasoner lacks the knowledge that this lead to the open of bluetooth. 

% Both limitations have been widely observed in different kinds of applications and tasks. To handle both problems, \sysname proactive probe states with the collaboration between the general reasoner and the domain-specific evaluator agents. Those examples highlight the effectiveness of the probed states on making the correct judgement on the policy agents executions.

Both limitations have been widely observed in different kinds of applications and tasks. To handle both problems, \sysname proactive probe states with the collaboration between the general reasoner and the domain-specific evaluator agents. Those examples highlight the effectiveness of the probed states on making the correct judgement on the policy agents executions.

\begin{figure}[!b]
    \centering
    \includegraphics[width=0.7\linewidth]{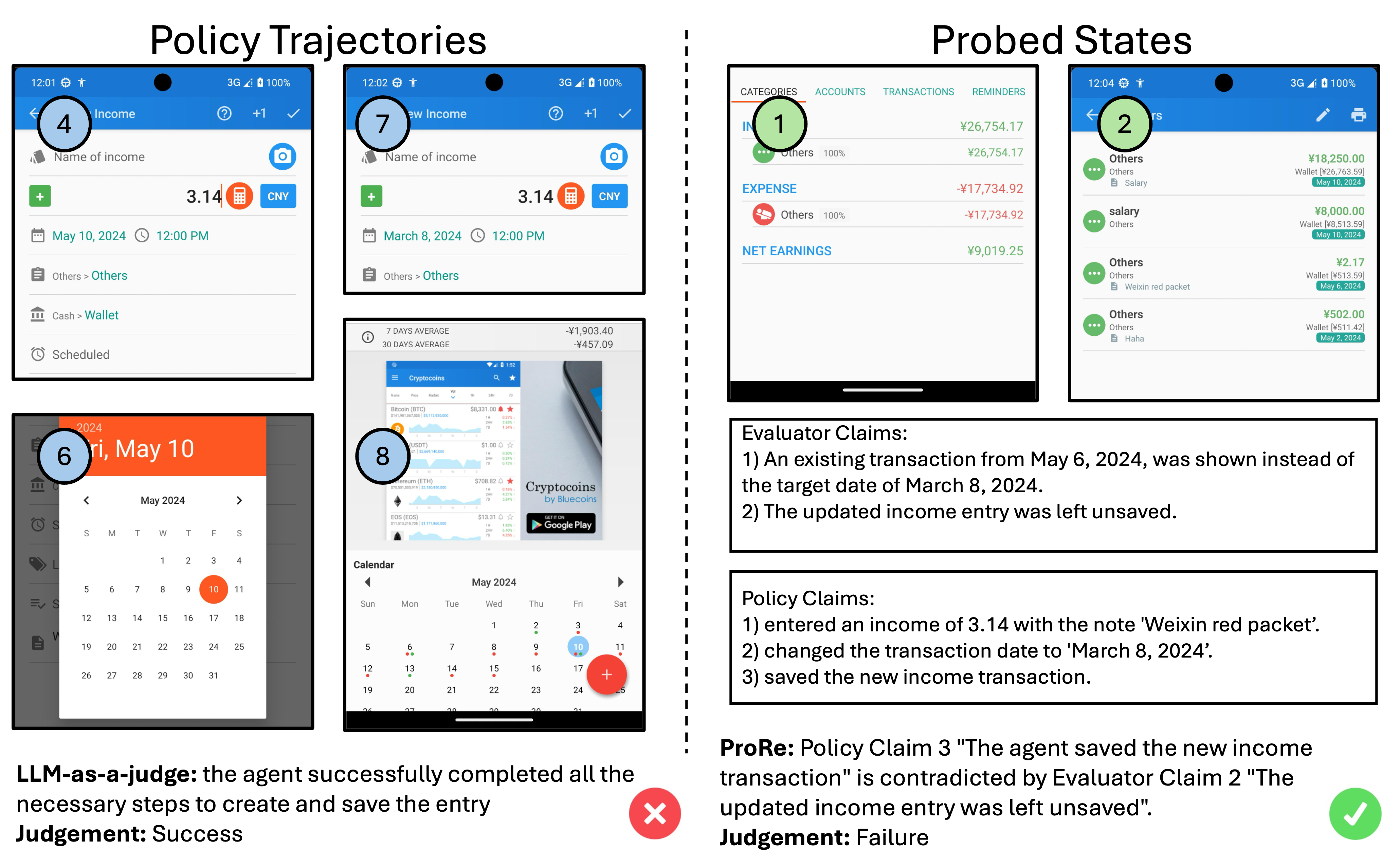}
    \caption{Additional examples. The task is "\textit{Switch the May 13, 2024, transaction from 'expense' to 'income' and add 'Gift' as the note in Bluecoins.}"}
    \label{fig:add_example1}
\end{figure}

\begin{figure}
    \centering
    \includegraphics[width=0.7\linewidth]{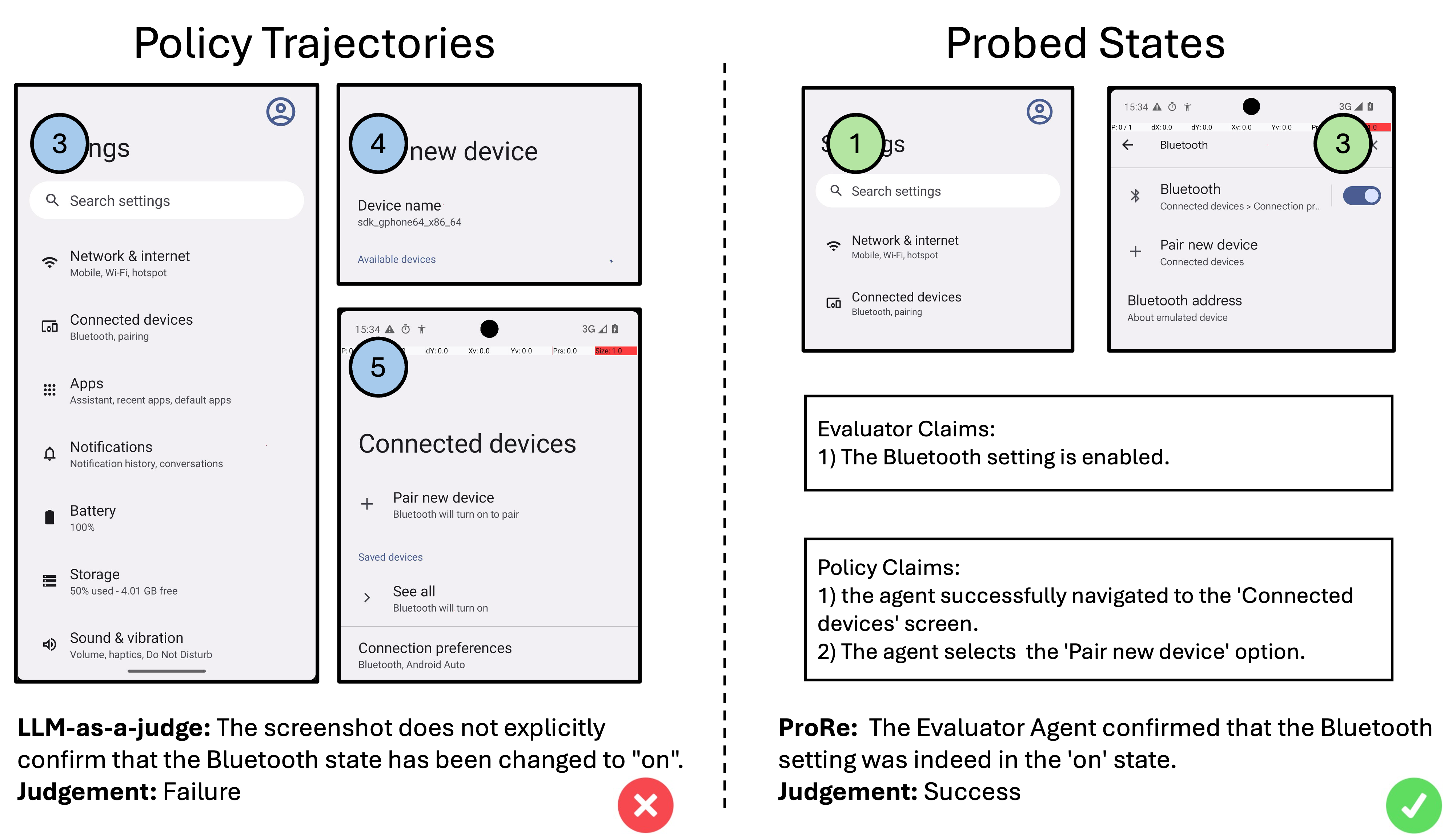}
    \caption{Additional examples. The task is "\textit{Turn bluetooth on.}"}
    \label{fig:add_example2}
\end{figure}

\begin{figure}
    \centering
    \includegraphics[width=0.7\linewidth]{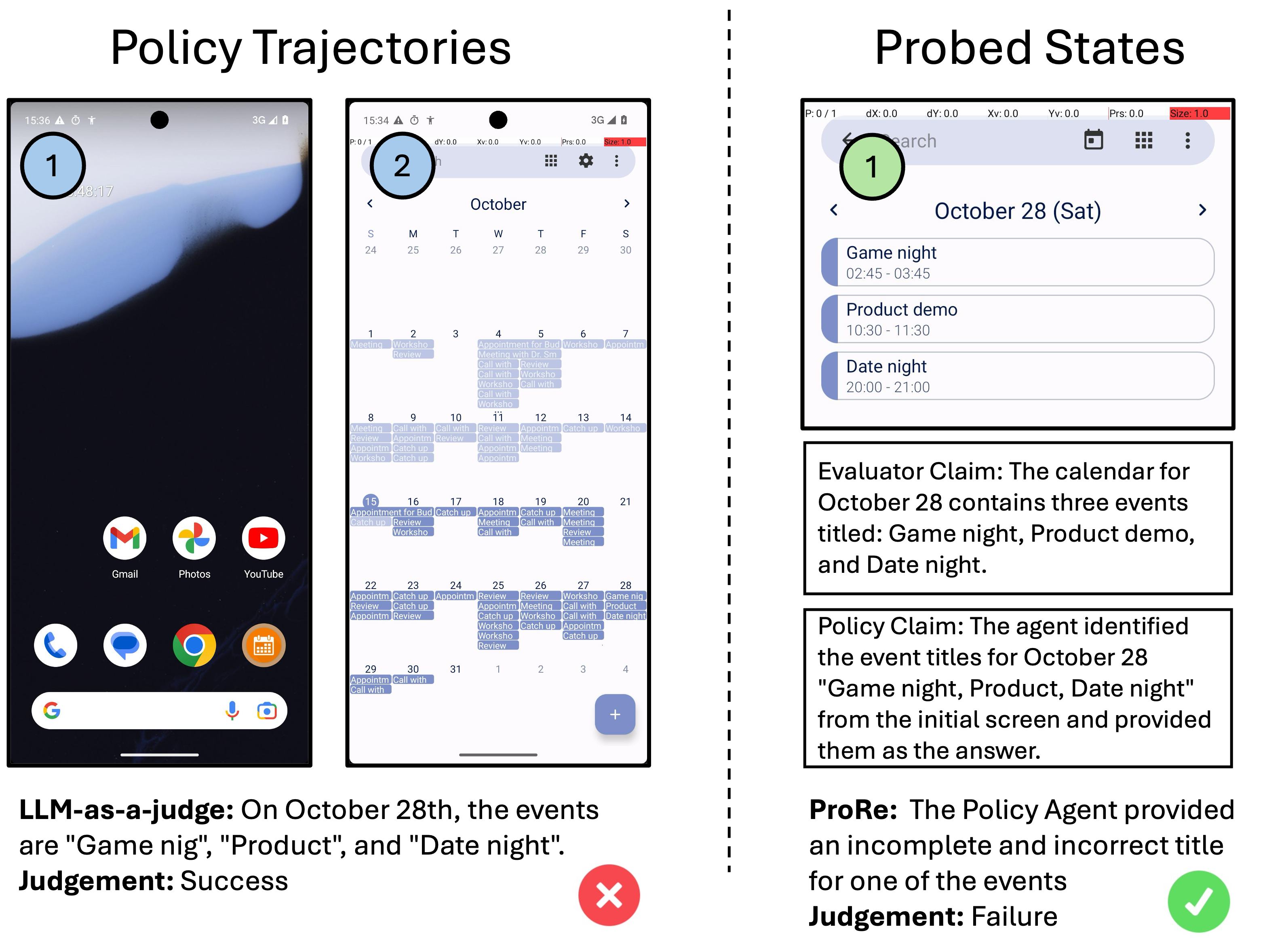}
    \caption{Additional examples. The task is "\textit{Do I have any events October 28 in Simple Calendar Pro? Answer with the titles only. If there are multiples titles, format your answer in a comma separated list.}"}
    \label{fig:add_example3}
\end{figure}

\begin{figure}
    \centering
    \includegraphics[width=0.7\linewidth]{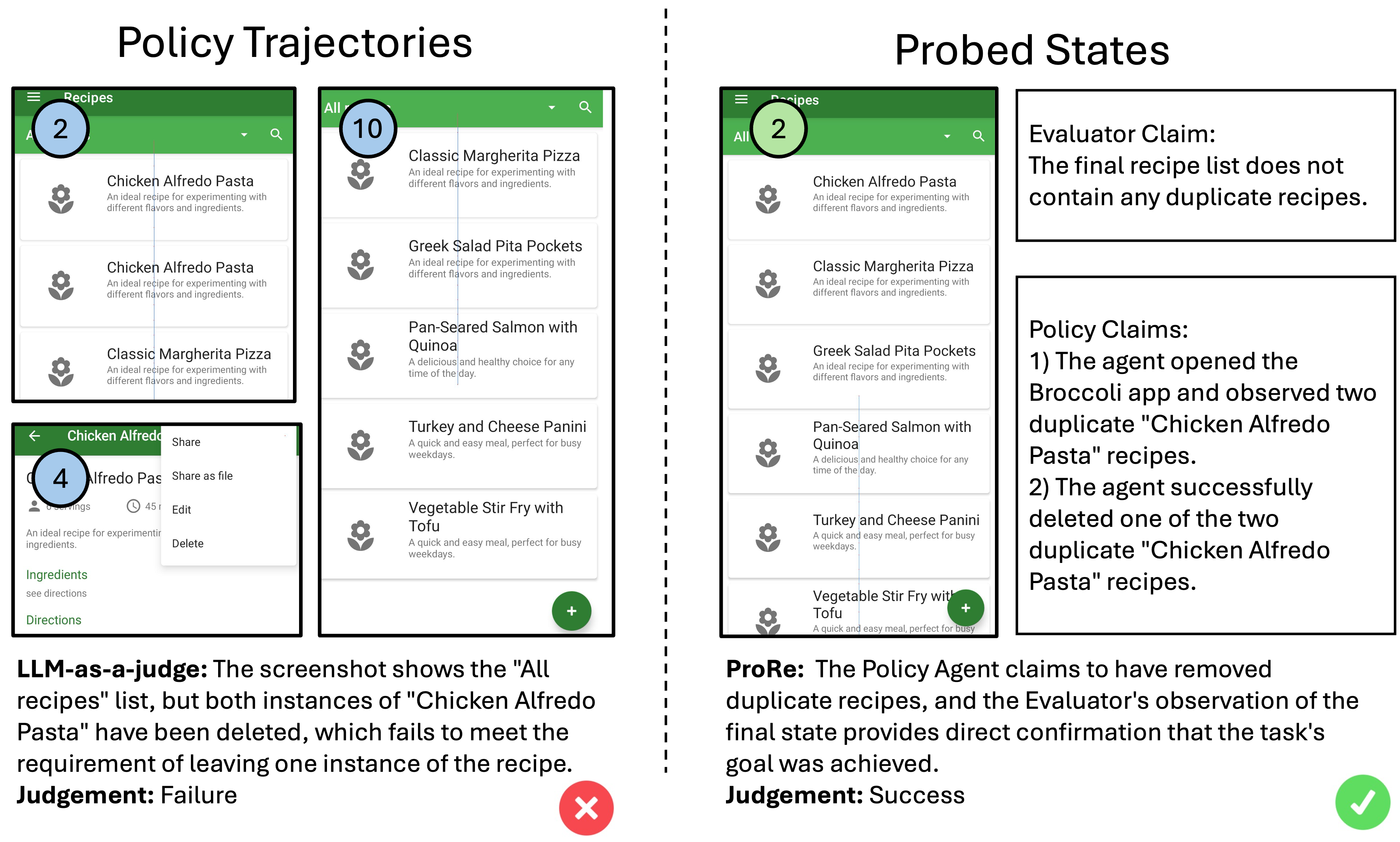}
    \caption{Additional examples. The task is "\textit{Delete all but one of any recipes in the Broccoli app that are exact duplicates, ensuring at least one instance of each unique recipe remains}"}
    \label{fig:add_example4}
\end{figure}

\section{Prompts of \sysname}
\label{appendix:prompts}
For reproducibility, we provide the prompts used in our experiments, covering probing task scheduling, claim generation, and the final judgment with chain-of-claims. During state probing, the Evaluator Agent is invoked with the original prompt format from prior works, but conditioned on the specific evaluation goal.

\subsection{Probing Tasks Scheduling}
\begin{lstlisting}[basicstyle=\ttfamily, breaklines=true]
You are an expert in mobile-UI task verification.
There are two agents:
- The Policy Agent already attempted the task.
- The Evaluator Agent ONLY navigates the UI to probe states about whether the task was or wasn't completed. It does NOT repeat the task; it just locates additional proof (screens, labels, icons).

Your job:
1. Write some analysis explaining what UI evidence/states would confirm the task is done.
2. Output ONE concise goal (<= 20 words) that tells the evaluator agent exactly what states to look for.
3. When the original task involves multiple key states, you may decompose the verification into a sequence of probing goals, with each goal focusing on a specific state.

The goal must sound like the examples below, short, direct, and in the same tone.

### Style Examples
"What is the cheapest flight from Los Angeles to Tokyo using Skyscanner?"
"What are the 1M to 3M GBP to EUR exchange rates?"
"go to settings and make weeks start on Monday in simple calendar"
"Mark Hamlet as read in Cantook."

### Your turn
Original task: {goal}

A previous state probing task was:
{previous_state_probing_task}

The Evaluator Agent probed the following states:
{collected_info}

Revise the probing task based on the previous probing task and the original task:
Respond exactly as:
Analysis: <outline the users expectations and exact UI evidence needed, pinpoint why the earlier collection failed, and suggest how to refine the evaluation goal for comprehensive verification>
Goal: <revised concise goal>
Do not add anything else.
\end{lstlisting}

\subsection{Chaims Generations}

\begin{lstlisting}[basicstyle=\ttfamily, breaklines=true, breakatwhitespace=true]
You are an expert in evaluating the performance of a mobile GUI agent.
**Workflow overview:**  
1. **User** provides a task intent.  
2. The **Policy Agent** executes UI actions to fulfil that task; its steps are recorded as *Action History*.  
3. The **Evaluator Agent** runs after the Policy Agent has finished, and proactively interact with the environment to gather additional observations.
4. **You** will now produce concise **claims** for the **{role.capitalize()} Agent** only.

You must follow a step-by-step analysis:
1. Read the **Task Goal** and the {role.capitalize()} Agent's action history (if available).
2. Examine the provided {role.capitalize()} screens (HTML + screenshots are attached in order).
3. Synthesize related observations into claims. Each claim must:
   - List the supporting step indices.
   - Give a brief, evidence-grounded rationale.
   - State a concise, goal-relevant claim.
4. Include any details critical to the final judgment directly in the claims (e.g., specific titles, timestamps, targets, confirmations, error messages).
6. Do **not** judge final success/failure here; only produce claims.

------  INPUTS  ------
TASK GOAL:
{intent}

ACTION HISTORY ({role.capitalize()} Agent):
{action_history if action_history else "[No action history provided]"}

HTML STATES (TRACE of {role.capitalize()} Agent): 
{html_text_block}

------ OUTPUT GUIDELINES ------
{guidelines}

------ OUTPUT SCHEMA ------
{{
  "{role_key}": [
    {{
      "steps": [<list of step numbers>],
      "reasoning": "<brief explanation of why this claim is justified>",
      "claim": "<concise, goal-relevant claim>"
    }},
    ...
  ]
}}

Return only the JSON under **Claims:**
\end{lstlisting}

Guidelines for Policy Agent to write claims:

\begin{lstlisting}[basicstyle=\ttfamily, breaklines=true, breakatwhitespace=true]
**Guidelines for writing claims (Policy Agent):**
### Core Mandate: The Actor's Report
- Think of yourself as the agent actively performing the task. Your claims are a direct report of your own actions and their immediate results.
- Your goal is to narrate your journey through the task, focusing only on the steps you took and the UI states you directly observed or caused.
- Be concise, factual, and strictly focused on the task goal. Avoid speculation or opinions about why something happened.

### Claim Generation Rules:
- Aim for {min_claims}-{max_claims} claims total.
- **For Tasks Involving Information Sources (e.g., "from an image," "using the details in the file"):**
  - **1. Access the Source:** Generate a claim confirming you **accessed and viewed the specified information source**.
    - *Example:* "The agent opened `expenses.jpg` in the gallery to view the expense details."
  - **2. Confirm Data Match:** In the claim about entering the data, explicitly state that the **data entered matches the data from the source**.
    - *Example:* "The two expenses entered, 'Office Supplies for 150' and 'Travel Expenses for 200', match the content of `expenses.jpg`."
- **For Editing, Modifying, or Deletion Tasks:**
  - **1. Capture the 'Before' State:** First, generate a claim that **describes the initial state of the item BEFORE the modification**.
    - *Example:* "Before editing, the contact's phone number was '555-123-4567'."
  - **2. Report the 'After' State:** Then, generate a separate claim describing the **successful modification or deletion**.
    - *Example:* "The contact's phone number was successfully updated to '555-987-6543'."
- **Report All Critical Actions:**
  - Describe your key actions and their direct consequences using state/action phrasing (e.g., "Recording saved and appears in list").
  - Highlight any mismatches, errors, or unintended actions you performed (e.g., "Opened the wrong menu," "A 'Permission Denied' error appeared").
- **Be Efficient and Relevant:**
  - Merge duplicate claims that describe the same state.
  - Ignore trivial system indicators (battery, clock, signal), home/launcher screens, and redundant repeated actions unless they are evidence of an error or loop.
- Output must be valid JSON following the schema below.
\end{lstlisting}

Guidelines for the Evaluator Agent to write claims.
\begin{lstlisting}[basicstyle=\ttfamily, breaklines=true, breakatwhitespace=true]
**Guidelines for writing claims (Evaluator Agent):**

### Core Mandate: The Detective Analogy
- Think of yourself as a detective arriving at a scene *after* the suspect (the Policy Agent) has left.
- The action history and screenshots you see are your own investigation, using your 'magnifying glass' and 'tools' to inspect the scene.
- Your goal is to make claims about the state of the scene **as the Policy Agent left it**.
- **You must NEVER create a claim about your own investigative actions.** For example, if you tap 'Save' or 'Delete' to check a confirmation dialog, you must not claim "The agent saved the file" or "The agent deleted the item." Your actions are not part of the evaluated task.

## Claim Generation Rules:
- Aim for {min_claims}-{max_claims} claims total.
- **Focus on the evidence:** All claims must describe the final state resulting from the Policy Agent's work, using your observations as proof.
- **Be factual and concise:** Merge duplicates and report on what is present or missing. Avoid speculation.
- **Identify mismatches:** If your investigation reveals that the final state contradicts the task goal (e.g., wrong file type, incorrect note name, content not saved, settings not changed), these are critical claims to include.
- **Ignore trivial states:** Do not report on system indicators (battery, clock), home screens, or app launchers unless directly relevant to the task goal.
- **Phrase claims effectively:** Prefer state/action summaries (e.g., "Recording saved and appears in list") over simple lists of UI elements.
- Output must be valid JSON following the schema below.
\end{lstlisting}

\subsection{Judge with Chain-of-Claims}
\begin{lstlisting}[basicstyle=\ttfamily, breaklines=true, breakatwhitespace=true]
You are an expert judge evaluating whether a mobile GUI agent (Policy Agent) has completed the user's task.  

**Workflow overview:**  
1. **User** provides a task intent.  
2. The **Policy Agent** executes UI actions to fulfil that task; its steps are recorded as *Action History*.  
3. The **Evaluator Agent** runs after the Policy Agent has finished, and proactively probes the resulting states to gather additional observations.
4. Your job is to analyze these claims together, identify their relationships, and determine whether the Policy Agent successfully completed the task.  

You must follow a two-stage analysis:  

### Stage 1 - Filter Evaluator Claims
- Carefully review the evaluator claims.  
- **Discard any claim that describes actions or outcomes caused by the Evaluator Agent itself** (e.g., accidental saves, unintended edits, stray taps/scrolls).  
- Keep only evaluator claims that serve as **evidence about the Policy Agent's actual outcome**.  
- If in doubt, prefer to exclude rather than include.  

### Stage 2 - Compare Policy vs. Evaluator Claims
1. **Read the Task Goal** carefully to understand what success means.  
2. **Compare Policy Claims and (filtered) Evaluator Claims**:  
   - Mark as **confirmed** if an evaluator claim supports a policy claim.  
   - Mark as **contradicted** if an evaluator claim directly disproves a policy claim.  
   - Mark as **complementary** if the evaluator provides additional relevant evidence.  
   - Mark as **unsupported** if no evaluator claim addresses a policy claim.  
3. Highlight any **critical confirmations or contradictions** that directly determine success.  
4. Decide the outcome reward: did the Policy Agent achieve the user's task goal?  

**Guidelines:**  
- Before labeling a contradiction, check if the agents are simply observing different aspects of the same content (e.g., Policy saw page 1, Evaluator scrolled to page 2).
  - If so, their claims are **complementary**. Your job is to **synthesize** them into a single, more complete understanding.
- When claims are in direct conflict, act as a critical arbiter rather than a passive matcher. Evaluate reliability and consistency; do not assume both sides are equally valid.
- Consider the correctness of the **target** (e.g., the right file, event, app).  
- For question-answer tasks, the Policy Agent must include an explicit claim stating the answer it provided, expressed exactly as required by the task.
- Ignore evaluator stray/accidental claims unrelated to the goal.  
- If claims indicate progress but also critical issues (wrong extension, malicious steps), treat as compensated or failure depending on severity.  

------  INPUTS  ------
TASK GOAL:  
{intent}

POLICY CLAIMS:  
{policy_claims}

EVALUATOR CLAIMS:  
{evaluator_claims}

------ OUTPUT INSTRUCTIONS ------  
Write your reasoning in two sections:  

Analysis:  
- Stage 1: List which evaluator claims you discarded and why.  
- Stage 2: Compare the remaining evaluator claims against the policy claims, showing relations (confirmed, contradicted, complementary, unsupported).  
- Explain how these relations support your judgment.  

Status: success or failure  

Return only these two sections, exactly in this format.  
\end{lstlisting}

\section{Cost Comparison} \label{appendix:cost_comparison}
We further analyze and compare the cost between \sysname and the baselines from per-task cost and the long-term cost perspective. The per-cost cost comparison is detailed in Table \ref{tab:cost_comparison}. 

\begin{table}[h]
\centering
\caption{Per-task cost comparison across baselines.}
\begin{tabular}{lccccc}
\toprule
Methods & DistRL & DigiRL & WebRL & StepCritic & \sysname \\
\midrule
Cost (\$) & 0.010 & 0.014 & 0.013 & 0.017 & 0.063 \\
\bottomrule
\end{tabular}
\label{tab:cost_comparison}
\end{table}

While \sysname incurs additional computational overhead, primarily due to proactive probing and the chain-of-claims mechanism, it achieves significantly higher reward accuracy and F1 scores, as demonstrated in Table \ref{tab:policy-agent-results} and Figure \ref{fig:avg_bench}. We believe that enhanced reward accuracy ultimately translates to greater efficiency when deploying such a reward system. Therefore, we conducted the additional measurements and analysis.

Firstly, the reward system can be employed to guide test-time scaling of policy agents. As illustrated in Fig. \ref{fig:test-time-scaling-aw}, \sysname facilitates markedly more efficient test-time scaling: the policy agent attains a 63\% success rate with only 2 trials under \sysname, whereas the baseline approaches require at least 5 trials, resulting in a $2.5\times$ speedup.

Secondly, the reward system can be employed during training to guide the roll-outs of policy agents. In this setting, the total cost consists of both the rollout cost and the reward evaluation cost. To quantify this, we estimate the cost of collecting 1,000 correctly identified successful trajectories.

A correctly identified successful trajectory requires both (i) a successful rollout by the policy agent, and (ii) the evaluator correctly recognizing it as success. Assuming a 60\% policy success rate and using the average accuracies in Table \ref{tab:policy-agent-results} (93.7\% for \sysname and 88.4\% for StepCritic), the probability of obtaining one useful trajectory is $0.60\times Acc$. Consequently, \sysname requires approximately 1,780 rollouts, whereas StepCritic requires 1,885 rollouts—i.e., StepCritic needs 110 additional rollouts to achieve the same amount of useful data. The corresponding evaluator costs are \$112.1 for \sysname and \$32.1 for StepCritic.

\begin{table}[th]
\centering
\caption{Long-term cost comparison (1,000 useful trajectories)}
\label{tab:cost_longterm}
\begin{tabular}{lccccc}
\toprule
Method & Avg Acc (\%) & Rollouts & Rollout Cost & Reward Cost & Total Cost \\
\midrule
\sysname      & 93.7 & 1778.7 & 1636.4 & 112.1 & 1748.5 \\
StepCritic & 88.4 & 1885.4 & 1734.5 & 32.1  & 1766.6 \\
WebRL      & 86.9 & 1917.9 & 1764.5 & 24.9  & 1789.4 \\
DistRL     & 86.1 & 1935.7 & 1780.9 & 19.4  & 1800.2 \\
DigiRL     & 84.6 & 1970.1 & 1812.5 & 27.6  & 1840.0 \\
\bottomrule
\end{tabular}
\end{table}

This leads to the following cost difference:

\[
Cost_{\text{StepCritic}} - Cost_{\text{\sysname}}
= 110 \times Cost_{\text{Rollout}} - 85.6
\]

\sysname becomes more economical once the rollout cost exceeds \$0.78. Using Azure A100 pricing (\(\approx\$3.67/\text{hour}\)), a typical 72B GUI agent rollout with 30 steps (30 s/step) costs roughly \$0.92, already above this threshold. Thus, under realistic deployment conditions where rollout cost dominates (GPU hosting, LLM inference, environment rendering), \sysname becomes more cost-effective for large-scale training and evolution, despite its higher per-task evaluation overhead.

Similar deductions apply to all other baselines. Table~\ref{tab:cost_longterm} summarizes the total cost of collecting 1{,}000 useful trajectories for each method.

\end{document}